\documentclass[twoside,11pt]{article}

\usepackage{jmlr2e}
\usepackage{mathrsfs,amsmath,amsfonts,amssymb,bm,bbm,eufrak,pifont,epsfig,algorithm,algorithmic,euscript,dsfont}

\newcommand{\C}{{\mathbb{C}}}
\newcommand{\R}{{\mathbb{R}}}

\newcommand{\bb}{\mathbb}

\firstpageno{1}

\begin{document}

\title{A D.C. Programming Approach to the Sparse Generalized Eigenvalue Problem}

\author{\name Bharath K. Sriperumbudur \email bharathsv@ucsd.edu \\
       \addr Department of Electrical and Computer Engineering\\
       University of California, San Diego\\
       La Jolla, CA 92093-0407, USA
       \AND
		 \name David A. Torres \email datorres@cs.ucsd.edu\\
		 \addr Department of Computer Science and Engineering\\
		 University of California, San Diego\\
       La Jolla, CA 92093-0407, USA
		 \AND
		 \name Gert R. G. Lanckriet \email gert@ece.ucsd.edu \\
       \addr Department of Electrical and Computer Engineering\\
       University of California, San Diego\\
       La Jolla, CA 92093-0407, USA
		 }

\editor{}

\maketitle

\begin{abstract}
In this paper, we consider the sparse eigenvalue problem wherein the goal is to obtain a sparse solution to the generalized eigenvalue problem. We achieve this by constraining the cardinality of the solution to the generalized eigenvalue problem and obtain sparse principal component analysis (PCA), sparse canonical correlation analysis (CCA) and sparse Fisher discriminant analysis (FDA) as special cases. Unlike the $\ell_1$-norm approximation to the cardinality constraint, which previous methods have used in the context of sparse PCA, 
we propose a tighter approximation that is related to the negative
log-likelihood of a Student's t-distribution. The problem is then
framed as a d.c. (difference of convex functions) program and is
solved as a sequence of convex programs by invoking the majorization-minimization method. The resulting algorithm is proved to exhibit \emph{global convergence} behavior, i.e., for any random initialization, the sequence (subsequence) of iterates generated by the algorithm converges to a stationary point of the d.c. program. The performance of the algorithm is empirically demonstrated on both sparse PCA (finding few relevant genes that explain as much variance as possible in a high-dimensional gene dataset) and sparse CCA (cross-language document retrieval and vocabulary selection for music retrieval) applications.
%
\end{abstract}
\section{Introduction}\label{Sec:Intro}
The generalized eigenvalue (GEV) problem for the matrix pair $(\bm{A},\bm{B})$ is the problem of finding a pair $(\lambda,\bm{x})$ such that
\begin{equation}\label{Eq:eig}
\bm{Ax}=\lambda\bm{Bx},
\end{equation}
where $\bm{A},\,\bm{B}\in\C^{n\times n}$, $\C^n\ni\bm{x}\ne\bm{0}$ and $\lambda\in\C$. When $\bm{B}$ is an identity matrix, the problem in (\ref{Eq:eig}) is simply referred to as an eigenvalue problem. Eigenvalue problems are so fundamental that they have applications in almost every area of science and engineering~\citep{Strang-86}.
\par In multivariate statistics, GEV problems are prominent and appear in problems dealing with high-dimensional data analysis, visualization and pattern recognition. In these applications, usually $\bm{x}\in\R^n$, $\bm{A}\in\mathbb{S}^n$ (the set of symmetric matrices of size $n\times n$ defined over $\R$) and $\bm{B}\in\mathbb{S}^n_{++}$ (set of positive definite matrices of size $n\times n$ defined over $\R$). The variational formulation for the GEV problem in (\ref{Eq:eig}) is given by
\begin{equation}\tag{GEV-P}
\lambda_{max}(\bm{A},\bm{B})=\max\{\bm{x}^T\bm{A}\bm{x}\,:\,\bm{x}^T\bm{B}\bm{x}=1\},
\end{equation}
where $\lambda_{max}(\bm{A},\bm{B})$ is the maximum generalized eigenvalue associated with the matrix pair, $(\bm{A},\bm{B})$. The $\bm{x}$ that maximizes (GEV-P) is called the generalized eigenvector associated with $\lambda_{max}(\bm{A},\bm{B})$. Some of the well-known and widely used data analysis techniques that are specific instances of (GEV-P) are: 
\begin{itemize}
\item[(a)] Principal component analysis (PCA)~\citep{Hotelling-33,Jollife-86}, a classic tool for data analysis, data compression and visualization, finds the direction of maximal variance in a given multivariate data set. This technique is used in dimensionality reduction wherein the ambient space in which the data resides is approximated by a low-dimensional subspace without significant loss of information. The variational form of PCA is obtained by choosing $\bm{A}$ to be the covariance matrix (which is a positive semidefinite matrix defined over $\R$) associated with the multivariate data and $\bm{B}$ to be the identity matrix in (GEV-P).
\item[(b)] Canonical correlation analysis (CCA)~\citep{Hotelling-36}, similar to PCA, is also a data analysis and dimensionality reduction method. However, while PCA deals with only one data space $\EuScript{X}$ (from which the multivariate data is obtained), CCA proposes a way for dimensionality reduction by taking into account relations between samples from two spaces $\EuScript{X}$ and $\EuScript{Y}$. The assumption is that the data points from these two spaces contain some joint information that is reflected in correlations between them. Directions along which this correlation is high are thus assumed
to be relevant directions when these relations are to be captured.
The variational formulation for CCA is given by
\begin{equation}\label{Eq:correlation}
\max_{\bm{w}_x\ne\bm{0},\,\bm{w}_y\ne\bm{0}}\,\frac{\bm{w}^T_x\bm{\Sigma}_{xy}\bm{w}_y}{\sqrt{\bm{w}^T_x\bm{\Sigma}_{xx}\bm{w}_x}\sqrt{\bm{w}^T_y\bm{\Sigma}_{yy}\bm{w}_y}},
\end{equation}
where $\bm{w}_x$ and $\bm{w}_y$ are the directions in $\EuScript{X}$ and $\EuScript{Y}$ along which the data is maximally correlated. $\bm{\Sigma}_{xx}$ and $\bm{\Sigma}_{yy}$ represent the covariance matrices for $\EuScript{X}$ and $\EuScript{Y}$ respectively and $\bm{\Sigma}_{xy}=\bm{\Sigma}^T_{yx}$ represents the cross-covariance matrix between $\EuScript{X}$ and $\EuScript{Y}$. (\ref{Eq:correlation}) can be rewritten as
\begin{equation} \max\{\bm{w}^T_x\bm{\Sigma}_{xy}\bm{w}_y\,:\,
\bm{w}^T_x\bm{\Sigma}_{xx}\bm{w}_x=1,\,\bm{w}^T_y\bm{\Sigma}_{yy}\bm{w}_y=1\},
\label{Eq:CCA}
\end{equation} 
which in turn can be written in the form of
(GEV-P) with $\bm{A}=\small{\left(\begin{array}{cc}
\bm{0} & \bm{\Sigma}_{xy} \\
\bm{\Sigma}_{yx}&\bm{0}\end{array}\right)}$,
$\bm{B}=\small{\left(\begin{array}{cc}
\bm{\Sigma}_{xx} & \bm{0} \\
\bm{0}&\bm{\Sigma}_{yy}\end{array}\right)}$ and
$\bm{x}=\small{\left(\begin{array}{c}\bm{w}_x \\
\bm{w}_y \end{array}\right)}$.
\item[(c)] In the binary classification setting, Fisher discriminant analysis (FDA) finds a one-dimensional subspace, $\bm{w}\in\R^n$, the projection of data onto which leads to maximal separation between the classes. Let $\bm{\mu}_i$ and $\bm{\Sigma}_i$ denote the mean vector and covariance matrix associated with class $i$. The variational formulation of FDA is given by
\begin{eqnarray}
\max_{\bm{w}\neq\bm{0}} &&
\frac{(\bm{w}^T(\bm{\mu}_1-\bm{\mu}_2))^2}{\bm{w}^T(\bm{\Sigma}_1+\bm{\Sigma}_2)\bm{w}},\nonumber
\end{eqnarray}
which can be rewritten as
\begin{eqnarray}
\max_{\bm{w}} && \bm{w}^T(\bm{\mu}_1-\bm{\mu}_2)(\bm{\mu}_1-\bm{\mu}_2)^T\bm{w}\nonumber\\
\text{s.t.}&&  \bm{w}^T(\bm{\Sigma}_1+\bm{\Sigma}_2)\bm{w}=1.
\label{Eq:FDA}
\end{eqnarray}
Therefore, the FDA formulation is similar to (GEV-P) with
$\bm{A}=(\bm{\mu}_1-\bm{\mu}_2)(\bm{\mu}_1-\bm{\mu}_2)^T$,
called the \emph{between-cluster variance} and
$\bm{B}=\bm{\Sigma}_1+\bm{\Sigma}_2$, called the \emph{within-cluster
variance}. For multi-class problems, similar formulations lead to multiple-discriminant analysis.
\end{itemize}
\par Despite the simplicity and popularity of these data analysis and modeling
methods, one key drawback is the lack of sparsity in their solution. They suffer from the disadvantage that their solution vector, i.e., $\bm{x}$ is a linear combination of all input  variables, which often makes it difficult to interpret the results. In the following, we point to different applications where PCA/CCA/FDA is used and motivate the need for sparse solutions. 
\par In many PCA applications, the coordinate axes have a physical interpretation; in biology, for example, each axis might correspond to a specific gene. In these cases, the interpretation of the principal components would be facilitated if they contained only few non-zero entries (or, loadings) while explaining most of the variance in the data. Moreover, in certain applications, e.g., financial asset trading strategies based on PCA techniques, the sparsity of the solution has important consequences, since fewer non-zero loadings imply fewer transaction costs. For CCA, consider a document translation application where two copies of a corpus of documents, one written in English and the other in German are given. The goal is to extract multiple low-dimensional representations of the documents, one in each language, each explaining most of the variation in the documents of a single language while maximizing the correlation between the representations to aid translation. Sparse representations, equivalent to representing the documents with a small set of words in each language, would allow to interpret the underlying translation mechanism and model it better.
In music annotation, CCA can be applied to model the correlation between semantic descriptions of songs (e.g., reviews) and their acoustic content. Sparsity in the semantic canonical components would allow to select the most meaningful words to describe musical content. This is expected to improve music annotation and retrieval systems.
In a classification setting like FDA, feature selection aids generalization performance by promoting sparse solutions. To summarize, sparse representations are generally desirable as they aid human understanding, reduce computational and economic costs and promote better generalization.
\par In this paper, we consider the problem of finding sparse solutions while explaining the statistical information in the data, which can be written as
\begin{equation}\tag{SGEV-P}
\max\{\bm{x}^T\bm{Ax}\,:\,\bm{x}^T\bm{Bx}=1,\,\Vert\bm{x}\Vert_0\le k\},
\end{equation}
where $1\le k\le n$ and $\Vert\bm{x}\Vert_0$ denotes the cardinality of $\bm{x}$, i.e., the number of non-zero elements of $\bm{x}$. The above program can be solved either as a continuous optimization problem after relaxing the cardinality constraint or as a discrete optimization problem. In this paper, we follow the former approach. The first step in solving (SGEV-P) as a continuous optimization problem is to approximate the cardinality constraint. One usual heuristic is to approximate $\Vert\bm{x}\Vert_0$ by $\Vert\bm{x}\Vert_1$ (see Section~\ref{Sec:Notation} for the details on notation). 
Building on the earlier version of our work~\citep{Sriperumbudur-07b}, in this paper, we approximate the cardinality constraint in (SGEV-P) as the negative log-likelihood of a Student's t-distribution, which has been used earlier in many different contexts \citep{Weston-02,Fazel-03,Candes-07}. We then formulate this approximate problem as a d.c. (difference of convex functions) program and solve it using the majorization-minimization (MM) method \citep{Hunter-04} resulting in a sequence of quadratically constrained quadratic programs (QCQPs). As a special case, when $\bm{A}$ is positive definite and $\bm{B}$ is an identity matrix (as is the case for PCA), a very simple iterative update rule (we call it as DC-PCA) can be obtained in a closed form, which has a per iteration complexity of $O(n^2)$. Since the proposed algorithm is an iterative procedure, using results from the global convergence theory of iterative algorithms~\citep{Zangwill-69}, we show that it is \emph{globally convergent}, i.e., for any random initialization, the sequence (subsequence) of iterates generated by the algorithm converges to a stationary point of the d.c. program (see Section~\ref{Sec:convergence} for a detailed definition). 
We would like to mention that the algorithm presented in this paper is more general than the one in \citet{Sriperumbudur-07b} as it holds for any $\bm{A}\in\mathbb{S}^n$ unlike in \citet{Sriperumbudur-07b}, where $\bm{A}$ is assumed to be positive semidefinite. 
\par We illustrate the performance of the proposed algorithm on sparse PCA and sparse CCA problems. On the sparse PCA front, we compare our results to SPCA \citep{Zou-06}, DSPCA \citep{Aspremont-07}, GSPCA \citep{Moghaddam-06a} and GPower$_{\ell_0}$ \citep{Journee-08} in terms of sparsity vs. explained variance on the ``pit props" benchmark dataset and a random test dataset. Since DSPCA and GSPCA are not scalable for large-scale problems, we compare the performance of DC-PCA to SPCA and GPower$_{\ell_0}$ on three high-dimensional gene datasets where the goal is to find relevant genes (as few as possible) while explaining the maximum possible variance. The results show that DC-PCA performs similar to most of these algorithms and better than SPCA, but at better computational speeds. The proposed sparse CCA algorithm is used in two sparse CCA applications, one dealing with cross-language document retrieval and the other with vocabulary selection in music annotation. The cross-language document retrieval application involves a collection of documents with each document in different languages, say English and French. The goal is, given a query string in one language, retrieve the most relevant document(s) in the target language. We experimentally show that the proposed sparse CCA algorithm performs similar to the non-sparse version, however using only 10\% of non-zero loadings 
in the canonical components. In the vocabulary selection application, we show that sparse CCA improves the performance of a statistical musical query system by selecting only those words (i.e., pruning the vocabulary) that are correlated to the underlying audio features.
\par The paper is organized as follows. We establish the mathematical notation in Section~\ref{Sec:Notation}. 
In Section~\ref{Sec:Sparse-GEV}, we present the sparse generalized eigenvalue problem and discuss a tractable convex semidefinite programming (SDP) approximation. Since the SDP approximation is computationally intensive for large $n$, in Section~\ref{Sec:DC}, we present our proposed approximation to the sparse GEV problem resulting in a d.c. program. This is then solved as a sequence of QCQPs in Section~\ref{Sec:algo} using the majorization-minimization method that is briefly discussed in Section~\ref{Sec:MM}. The convergence analysis of the sparse GEV algorithm is presented in Section~\ref{Sec:convergence}. Finally, in Sections~\ref{sec:pca} and \ref{sec:cca}, we derive sparse PCA and sparse CCA as special instances of the proposed algorithm and present experimental results to demonstrate the performance, while in Section~\ref{Sec:sparseFDA}, we discuss the applicability of the proposed algorithm to the sparse FDA problem.
\section{Notation}\label{Sec:Notation} 
$\mathbb{S}^n$ (respectively $\mathbb{S}^n_+$, $\mathbb{S}^n_{++}$) denotes the set of symmetric (respectively positive semidefinite, positive definite) $n\times n$ matrices defined over $\R$. For $\bm{X}\in\mathbb{S}^n$, $\bm{X}\succ 0$ (respectively $\bm{X}\succeq 0$) means that
$\bm{X}$ is positive definite (respectively semidefinite). We denote a vector of ones and zeros by $\bm{1}$ and $\bm{0}$ respectively. Depending on the context, $\bm{0}$ will also be treated as a zero matrix. $|\bm{X}|$ is the matrix whose elements are the absolute values of the elements of $\bm{X}$. $[\bm{X}]_{ij}$ denotes the $(i,j)^{th}$ element of $\bm{X}$. For $\bm{x}=(x_1,x_2,\ldots,x_n)^T\in\mathbb{R}^n$, $\bm{x}\succeq \bm{0}$ denotes an element-wise inequality. $\Vert\bm{x}\Vert_0$ denotes the number of non-zero elements of the vector $\bm{x}$, $\Vert\bm{x}\Vert_p:=(\sum^n_{i=1}|x_i|^p)^{1/p},\,1\le p<\infty$ and $\Vert\bm{x}\Vert_\infty:=\max_{1\le i\le n}|x_i|$. $\bm{I}_n$ denotes an $n\times n$ identity matrix. $\bm{D}(\bm{x})$ represents a diagonal matrix
formed with $\bm{x}$ as its principal diagonal. 

\section{Sparse Generalized Eigenvalue Problem}\label{Sec:Sparse-GEV}
As mentioned in Section~\ref{Sec:Intro}, the sparse generalized eigenvalue problem in (SGEV-P) can be solved either as a continuous optimization problem after relaxing the cardinality constraint or as a discrete optimization problem. In this section, we consider the former approach.  
\par Let us consider the variational formulation for the sparse generalized eigenvalue problem in (SGEV-P), where $\bm{A}\in\bb{S}^n$ and $\bm{B}\in\bb{S}^n_{++}$. Suppose $\bm{A}$ is not negative definite. Then (SGEV-P) is the maximization of a non-concave objective over the non-convex constraint set $\Phi:=\{\bm{x}:\bm{x}^T\bm{Bx}=1\}\cap\{\bm{x}:\Vert\bm{x}\Vert_0\le k\}$. Although $\Phi$ can be relaxed to a convex set $\widetilde{\Phi}:=\{\bm{x}:\bm{x}^T\bm{Bx}\le 1\}\cap\{\bm{x}:\Vert\bm{x}\Vert_1\le k\}$, it does not simplify the problem as the maximization of a non-concave objective over a convex set is still computationally hard and intractable [p. 342]\citep{Rockafeller-70}.\footnote{Note that (GEV-P) also involves the maximization of a non-concave objective over a non-convex set. However, it is well-known that polynomial-time algorithms exist to solve (GEV-P), which is due to its \emph{special} structure of a quadratic objective with a homogeneous quadratic constraint \citep[p. 229]{Boyd-06}.} So, the intractability of (SGEV-P) is due to two reasons: (a) maximization of the non-concave objective function and (b) the constraint set being non-convex. Since (SGEV-P) is intractable, instead of solving it directly, one can solve approximations to (SGEV-P) that are tractable. Different tractable approximations to (SGEV-P) are possible, of which we briefly discuss the convex semidefinite programming (SDP) approximation and then motivate our proposed non-convex approximation.
\par First, let us consider the following approximate program that is obtained by relaxing the non-convex constraint set $\Phi$ to the convex set $\widetilde{\Phi}$, as described before:
\begin{equation}\label{Eq:approx}
\max\{\bm{x}^T\bm{Ax}\,:\,\bm{x}\in\widetilde{\Phi}\}.\end{equation}
As mentioned before, this program is still intractable due to the maximization of the non-concave objective. Had the objective function been linear, (\ref{Eq:approx}) would have been a canonical convex program, which could then be solved efficiently. One approach to linearize the objective function is by using the \emph{lifting} technique \citep[Section 4.4]{Lemarechal-99}, which was considered by \citet{Aspremont-04} when $\bm{A}\succeq 0$ and $\bm{B}=\bm{I}_n$. The lifted version of (\ref{Eq:approx}) is given by (see Appendix A for details):
\begin{eqnarray}\label{Eq:Sparse-lifting-GEV}
\max_{\bm{X},\bm{x}}&& \text{tr}(\bm{XA})\nonumber\\
\text{s.t.}&& \text{tr}(\bm{XB})\le 1,\,\Vert\bm{x}\Vert_1\le k\nonumber \\
&& \bm{X}=\bm{xx}^T.
\end{eqnarray}
Note that in the above program, the objective function is linear in $\bm{X}$, and the constraints are convex except for the non-convex constraint,
$\bm{X}=\bm{xx}^T$ ($\bm{X}=\bm{xx}^T\,\Leftrightarrow\,\bm{X}\succeq 0,\,\text{rank}(\bm{X})=1$, where $\text{rank}(\bm{X})=1$ is a non-convex constraint and therefore $\bm{X}=\bm{xx}^T$ is a non-convex constraint). Relaxing $\bm{X}=\bm{xx}^T$ to $\bm{X}-\bm{xx}^T\succeq 0$ results in the following program
\begin{eqnarray}\label{Eq:Sparse-sdp-GEV}
\max_{\bm{X},\bm{x}}&& \text{tr}(\bm{XA})\nonumber\\
\text{s.t.}&& \text{tr}(\bm{XB})\le 1,\,\Vert\bm{x}\Vert_1\le k\nonumber \\
&& \bm{X}-\bm{xx}^T\succeq 0,
\end{eqnarray}
which is a semidefinite program (SDP) \citep{Vandenberghe-96}. 
The $\ell_1$-norm constraint in (\ref{Eq:Sparse-sdp-GEV}) can be relaxed as
$\Vert\bm{x}\Vert^2_1\le k^2\,\Rightarrow\,\bm{1}^T|\bm{X}|\bm{1}\le k^2$ so that the problem reduces to solving only for $\bm{X}$. Therefore, we have obtained a tractable convex approximation to (SGEV-P).
\par Although (\ref{Eq:Sparse-sdp-GEV}) is a \emph{convex} approximation to (SGEV-P), it is computationally very intensive as general purpose interior-point methods for SDP scale as $O(n^6\log\epsilon^{-1})$, where $\epsilon$ is the required accuracy on the optimal value. For large-scale problems,
first-order methods~\citep{Nesterov-05,Aspremont-07} can be used which scale as $O(\epsilon^{-1}n^4\sqrt{\log n})$. Therefore, the SDP-based convex relaxation to (SGEV-P) is prohibitively expensive in computation for large $n$.
\par In the following section, we propose a different approximation to (SGEV-P), wherein instead of the $\ell_1$-approximation to the cardinality constraint, we consider a non-convex approximation to it. We present a d.c. (difference of convex functions) formulation for this approximation to (SGEV-P), which is then solved as a sequence of QCQPs using the majorization-minimization algorithm.

\subsection{Non-convex approximation to $\Vert\bm{x}\Vert_0$ and d.c. formulation}\label{Sec:DC}
The proposed approximation to (SGEV-P) is motivated by the following observations. 
\begin{itemize}
\item Because of the non-concave maximization, a convex relaxation of the cardinality constraint does not simplify (SGEV-P). So, a better approximation to the cardinality constraint than the tightest convex relaxation, i.e., $\Vert\bm{x}\Vert_1$, can be explored to improve sparsity.
\item Approximations that yield good scalability should be explored (as opposed to, e.g., the SDP approximation which scales badly in $n$).
\end{itemize}
To this end, we consider the regularized
 (penalized) version of (SGEV-P) given by
\begin{equation}\tag{SGEV-R}
\max\{\bm{x}^T\bm{A}\bm{x}-\tilde{\rho}\,\Vert\bm{x}\Vert_0\,:\,\bm{x}^T\bm{B}\bm{x}\le 1\},
\end{equation}
where $\tilde{\rho}>0$ is the regularization (penalization) parameter. Note that the quadratic equality constraint, $\bm{x}^T\bm{Bx}=1$ is relaxed to the inequality constraint, $\bm{x}^T\bm{Bx}\le 1$. Since
\begin{equation}\label{Eq:lim-cardinality}
\Vert\bm{x}\Vert_0=\sum^n_{i=1}\mathds{1}_{\{|x_i|\ne 0\}}=\lim_{\varepsilon\rightarrow 0}\sum^n_{i=1}\frac{\log(1+|x_i|/\varepsilon)}{\log(1+1/\varepsilon)},
\end{equation}
(SGEV-R) is equivalent\footnote{Two programs are equivalent if their optimizers are the same.} to
\begin{eqnarray} 
\max_{\bm{x}}&&\bm{x}^T\bm{A}\bm{x}-\tilde{\rho}\lim_{\varepsilon\rightarrow 0}\sum^n_{i=1}\frac{\log(1+|x_i|/\varepsilon)}{\log(1+1/\varepsilon)}\nonumber\\
\text{s.t.}&&\bm{x}^T\bm{B}\bm{x}\le 1.
\label{Eq:GEV-sparse-mod-1}
\end{eqnarray}
The above program is approximated by the following \emph{approximate sparse GEV program} by neglecting the limit in (\ref{Eq:GEV-sparse-mod-1}) and choosing $\varepsilon>0$,
\begin{eqnarray} 
\max_{\bm{x}}&&\bm{x}^T\bm{A}\bm{x}-\tilde{\rho}\sum^n_{i=1}\frac{\log(1+|x_i|/\varepsilon)}{\log(1+1/\varepsilon)}\nonumber\\
\text{s.t.}&&\bm{x}^T\bm{B}\bm{x}\le 1,
\label{Eq:GEV-sparse-approx-1}
\end{eqnarray}
which is equivalent to
\begin{equation}\tag{SGEV-A} 
\max\left\{\bm{x}^T\bm{A}\bm{x}-\rho_\varepsilon\sum^n_{i=1}\log(|x_i|+\varepsilon)\,:\,
\bm{x}^T\bm{B}\bm{x}\le 1\right\},
\end{equation}
where $\rho_\varepsilon:=\tilde{\rho}/\log(1+\varepsilon^{-1})$. Note that the approximate program in (SGEV-A) is a continuous optimization problem unlike the one in (SGEV-R), which has a combinatorial term. Before we present a d.c. program formulation to (SGEV-A), we briefly discuss the approximation to $\Vert\bm{x}\Vert_0$ that we considered in this paper.\vspace{2mm}\\
\textbf{Approximation to $\Vert\bm{x}\Vert_0$:} The approximation (to $\Vert\bm{x}\Vert_0$) that we considered in this paper, i.e., 
\begin{equation}
\Vert\bm{x}\Vert_\varepsilon:=\sum^n_{i=1}\frac{\log(1+|x_i|\varepsilon^{-1})}{\log(1+\varepsilon^{-1})},\nonumber \end{equation}
has been used in many different contexts: feature selection using support vector machines \citep{Weston-02}, sparse signal recovery \citep{Candes-07}, matrix rank minimization \citep{Fazel-03}, etc. This approximation is interesting because of its connection to sparse factorial priors that are studied in Bayesian inference, and can be interpreted as defining a Student's t-distribution prior over $\bm{x}$, an improper prior given by $\prod^n_{i=1}\frac{1}{|x_i|+\varepsilon}$. \citet{Tipping-01} showed that this choice of prior leads to a sparse representation and demonstrated its validity
for sparse kernel expansions in the Bayesian framework. Other approximations to $\Vert\bm{x}\Vert_0$ are possible, e.g., \citet{Bradley-98} used $\sum^n_{i=1}(1-e^{-\alpha|x_i|})$ with $\alpha>0$ ($\Vert\bm{x}\Vert_0=\lim_{\alpha\rightarrow\infty}\sum^n_{i=1}(1-e^{-\alpha|x_i|})$) as an approximation to $\Vert\bm{x}\Vert_0$ in the context of feature selection using support vector machines. 
 \par We now show that the approximation (to $\Vert\bm{x}\Vert_0$) considered in this paper, i.e., $\Vert\bm{x}\Vert_\varepsilon$, is tighter than the $\ell_1$-norm approximation, for any $\varepsilon>0$. To this end, let us define \begin{equation}
a_\varepsilon:=\frac{\log(1+a\varepsilon^{-1})}{\log(1+\varepsilon^{-1})},\nonumber
\end{equation}
where $a\ge 0$, so that $\Vert\bm{x}\Vert_\varepsilon=\sum^n_{i=1}|x_i|_\varepsilon$. It is easy to check that $\Vert\bm{x}\Vert_0=\lim_{\varepsilon\rightarrow 0}\Vert\bm{x}\Vert_\varepsilon$ and $\Vert\bm{x}\Vert_1=\lim_{\varepsilon\rightarrow \infty}\Vert\bm{x}\Vert_\varepsilon$. In addition, we have $a>a_{\varepsilon_1}>a_{\varepsilon_2}>\ldots>1$ for $a>1$ and $1>\ldots>a_{\varepsilon_2}>a_{\varepsilon_1}>a$ for $0<a<1$, if $\varepsilon_1>\varepsilon_2>\ldots$, i.e., for any $a>0$ and any $0<\varepsilon<\infty$, the value $a_{\varepsilon}$ is closer to $1$ than $a$ is to $1$. This means $a_\varepsilon$ for any $0<\varepsilon<\infty$ is a better approximation to $\mathds{1}_{\{a\ne 0\}}$ than $a$ is to $\mathds{1}_{\{a\ne 0\}}$. Therefore, $\Vert\bm{x}\Vert_\varepsilon$ for any $0<\varepsilon<\infty$ is a better approximation to $\Vert\bm{x}\Vert_0$ than $\Vert\bm{x}\Vert_1$ is to $\Vert\bm{x}\Vert_0$.
%
%
\par Let us define 
\begin{eqnarray}
Q(\bm{x})&:=&\bm{x}^T\bm{Ax}-\tilde{\rho}\Vert\bm{x}\Vert_0,\nonumber\\ Q_\varepsilon(\bm{x})&:=&\bm{x}^T\bm{Ax}-\rho_\varepsilon\sum^n_{i=1}\log(1+|x_i|\varepsilon^{-1}) \nonumber
\end{eqnarray}
and $\Omega:=\{\bm{x}:\bm{x}^T\bm{Bx}\le 1\}$. Note that $Q(\bm{x})=\lim_{\varepsilon\rightarrow 0}Q_\varepsilon(\bm{x})$ for any fixed $\bm{x}$, i.e., $Q_\varepsilon$ converges pointwise to $Q$. So, the sparse GEV problem is obtained as $\max\{\lim_{\varepsilon\rightarrow 0}Q_\varepsilon(\bm{x}):\bm{x}\in\Omega\}$, while the approximate problem is given by $\max\{Q_\varepsilon(\bm{x}):\bm{x}\in\Omega\}$. Suppose that $\widehat{\bm{x}}$ denotes a maximizer of $Q(\bm{x})$ over $\Omega$ and $\bm{x}_\varepsilon$ denotes a maximizer of $Q_\varepsilon(\bm{x})$ over $\Omega$. Now, one would like to know how good is the approximate solution, $\bm{x}_\varepsilon$ compared to $\widehat{\bm{x}}$. In general, it is not straightforward to either bound $\Vert \bm{x}_\varepsilon-\widehat{\bm{x}}\Vert$ in terms of $\varepsilon$ or show that $\Vert \bm{x}_\varepsilon-\widehat{\bm{x}}\Vert\rightarrow 0$ as $\varepsilon\rightarrow 0$ because $Q(\bm{x})$ may be quite flat near its maximum over $\Omega$. 
At least, one would like to know whether $Q_\varepsilon(\bm{x}_\varepsilon)\rightarrow Q(\widehat{\bm{x}})$ as $\varepsilon\rightarrow 0$, i.e.,
\begin{equation}\label{Eq:question}
\lim_{\varepsilon\rightarrow 0}\max_{\bm{x}\in\Omega}Q_\varepsilon(\bm{x})\stackrel{?}{=}\max_{\bm{x}\in\Omega}Q(\bm{x})=\max_{\bm{x}\in\Omega}\lim_{\varepsilon\rightarrow 0}Q_\varepsilon(\bm{x}).
\end{equation}
In other words, we would like to know whether the limit process and the maximization over $\Omega$ can be interchanged. It can be shown that if $Q_\varepsilon$ converges uniformly over $\Omega$ to $Q$, then the equality in (\ref{Eq:question}) holds. However, it is easy to see that $Q_\varepsilon$ does not converge uniformly to $Q$ over $\Omega$, so nothing can be said about (\ref{Eq:question}). 
\vspace{2mm}\\ 
\textbf{D.c. formulation:} Let us return to the formulation in (SGEV-A). To solve this continuous, non-convex optimization problem and derive an algorithm for the sparse GEV problem, we explore its formulation as a d.c. program. D.c. programs are well studied and many algorithms exist to solve them \citep{Horst-99}. They are defined as follows.
\begin{definition}[D.c. program]\label{def:dc}
Let $\Omega$ be a convex set in $\mathbb{R}^n$. A real valued function
$f:\Omega\rightarrow\mathbb{R}$ is called a d.c. function on
$\Omega$, if there exist two \emph{convex} functions
$g,h:\Omega\rightarrow\mathbb{R}$ such that $f$ can be
expressed in the form $f(\bm{x})=g(\bm{x})-h(\bm{x}),\,
\bm{x}\in\Omega$. Optimization problems of the form
$\min\{f_0(\bm{x})\,:\,\bm{x}\in\Omega,f_i(\bm{x})\le
0,\,i=1,\ldots,m\}$, where $f_i=g_i-h_i,\,i=0,\ldots,m$, are
d.c.~functions are called \emph{d.c.~programs}.
\end{definition}
To formulate (SGEV-A) as a d.c. program, let us choose $\tau\in\mathbb{R}$ such that $\bm{A}+\tau\bm{I}_n\succeq 0$. If $\bm{A}\succeq 0$, such $\tau$ exists trivially (choose $\tau\ge 0$). If $\bm{A}$ is indefinite, choosing $\tau\ge-\lambda_{min}(\bm{A})$ ensures that $\bm{A}+\tau\bm{I}_n\succeq 0$. 
Therefore, choosing $\tau\ge\max(0,-\lambda_{min}(\bm{A}))$ ensures that $\bm{A}+\tau\bm{I}_n\succeq 0$ for any $\bm{A}\in\mathbb{S}^n$. (SGEV-A) is equivalently written as
\begin{eqnarray}\label{Eq:tau}
\min_{\bm{x}}&&\left[\tau\Vert \bm{x}\Vert^2_2-\bm{x}^T(\bm{A}+\tau\bm{I}_n)\bm{x}\right]+\rho_\varepsilon\sum^n_{i=1}\log(|x_i|+\varepsilon)\nonumber\\
\text{s.t.}&&\bm{x}^T\bm{Bx}\le 1.
\end{eqnarray}
Introducing the auxiliary variable, $\bm{y}$, yields the equivalent program 
\begin{eqnarray}\label{Eq:tau-1}
\min_{\bm{x},\bm{y}}&&\tau\Vert \bm{x}\Vert^2_2-\left[\bm{x}^T(\bm{A}+\tau\bm{I}_n)\bm{x}-\rho_\varepsilon\sum^n_{i=1}\log(y_i+\varepsilon)\right]\nonumber\\
\text{s.t.}&&\bm{x}^T\bm{Bx}\le 1,\,-\bm{y}\preceq\bm{x}\preceq\bm{y},
\end{eqnarray}
which is a d.c. program. Indeed, the term $\tau\Vert\bm{x}\Vert^2_2$ is convex in $\bm{x}$ as $\tau\ge 0$ and, by construction, $\bm{x}^T(\bm{A}+\tau\bm{I}_n)\bm{x}-\rho_\varepsilon\sum^n_{i=1}\log(y_i+\varepsilon)$ is jointly convex in $\bm{x}$ and $\bm{y}$. So, the above program is a minimization of the difference of two convex functions over a convex set. Global optimization methods like branch and bound, and cutting planes can be used to solve d.c. programs~\citep{Horst-99}, but are not scalable to large-scale problems. Since (\ref{Eq:tau-1}) is a constrained nonlinear optimization problem, it can be solved by, e.g., sequential quadratic programming, augmented Lagrangian methods or reduced-gradient methods~\citep{Bonnans-06}. In the following sections, we present an iterative algorithm to solve (\ref{Eq:tau-1}) using the majorization-minimization method.
\subsection{Majorization-minimization method}\label{Sec:MM}
The majorization-minimization (MM) method can be thought of as a generalization of the well-known expectation-maximization (EM) algorithm \citep{Dempster-77}. The general principle behind MM algorithms was first enunciated by the numerical analysts \citet{Ortega-70} in the context of line search methods. The MM principle appears in many places in statistical computation, including multidimensional scaling \citep{deLeeuw-77}, robust regression \citep{Huber-81}, correspondence analysis \citep{Heiser-87}, variable selection \citep{Hunter-05}, sparse signal recovery \citep{Candes-07}, etc. We refer the interested reader to a tutorial on MM algorithms \citep{Hunter-04} and the references therein.
\par The general idea of MM algorithms is as follows. Suppose we want to minimize $f$ over $\Omega\subset\mathbb{R}^n$. The idea is to construct a \emph{majorization function} $g$ over $\Omega\times\Omega$ such that
\begin{equation}\label{Eq:maj-step}
f(x)\le g(x,y),\,\forall\,x,y\in\Omega\qquad\text{and}\qquad f(x)=g(x,x),\,\forall\,x\in\Omega.
\end{equation}
Thus, $g$ as a function of $x$ is an upper bound on $f$ and coincides with $f$ at $y$. The majorization-minimization algorithm corresponding to this majorization function $g$ updates $x$ at iteration $l$ by
\begin{equation}\label{Eq:min-step}
x^{(l+1)}\in\arg\min_{x\in\Omega}g(x,x^{(l)}),
\end{equation}
unless we already have
\begin{equation}
x^{(l)}\in\arg\min_{x\in\Omega}g(x,x^{(l)}),\nonumber
\end{equation}
in which case the algorithm stops. The majorization function, $g$ is usually constructed by using Jensen's inequality for convex functions, the first-order Taylor approximation or the quadratic upper bound principle \citep{Bohning-88}. However, any other method can be used to construct $g$ as long as it satisfies (\ref{Eq:maj-step}). It is easy to show that the above iterative scheme decreases the value of $f$ monotonically in each iteration, i.e.,
\begin{equation}\label{Eq:sandwich}
f(x^{(l+1)})\le g(x^{(l+1)},x^{(l)})\le g(x^{(l)},x^{(l)})=f(x^{(l)}),
\end{equation}
where the first inequality and the last equality follow from (\ref{Eq:maj-step}) while the sandwiched inequality follows from (\ref{Eq:min-step}).
\par Note that MM algorithms can be applied equally well to the maximization of $f$ by simply reversing the inequality sign in (\ref{Eq:maj-step}) and changing the ``min" to ``max" in (\ref{Eq:min-step}). In this case, the word MM refers to minorization-maximization, where the function $g$ is called the \emph{minorization function.} To put things in perspective, the EM algorithm can be obtained by constructing the minorization function $g$ using Jensen's inequality for concave functions. The construction of such $g$ is referred to as the E-step, while (\ref{Eq:min-step}) with the ``min" replaced by ``max" is referred to as the M-step. The algorithm in (\ref{Eq:maj-step}) and (\ref{Eq:min-step}) is used in machine learning, e.g., for non-negative matrix factorization \citep{LeeSeung-01}, under the name \emph{auxiliary function method}. \citet{Lange-00} studied this algorithm under the name \emph{optimization transfer} while \citet{Meng-00} referred to it as the SM algorithm, where ``S" stands for the surrogate step (same as the majorization/minorization step) and ``M" stands for the minimization/maximization step depending on the problem at hand. $g$ is called the surrogate function. 
In the following, we consider an example that is relevant to our problem where we construct a majorization function, $g$, which will later be used in deriving the sparse GEV algorithm.
\begin{example}[Linear Majorization]\label{Exm:dc}
Let us consider the optimization problem, $\min_{\bm{x}\in\Omega}f(\bm{x})$ where $f=u-v$, with $u$ and $v$ both convex, and $v$ continuously differentiable. Since $v$ is convex, we have $v(\bm{x})\ge v(\bm{y})+(\bm{x}-\bm{y})^T\nabla v(\bm{y}),\,\forall\,\bm{x},\bm{y}\in\Omega$. Therefore, 
\begin{equation}\label{Eq:auxiliary-dc}
f(\bm{x})\le u(\bm{x})-v(\bm{y})-(\bm{x}-\bm{y})^T\nabla v(\bm{y})=:g(\bm{x},\bm{y}).
\end{equation}
It is easy to verify that $g$ is a majorization function of $f$. Therefore, we have
\begin{equation}\label{Eq:example}
\bm{x}^{(l+1)}\in\arg\min_{\bm{x}\in\Omega}\,g(\bm{x},\bm{x}^{(l)})=\arg\min_{\bm{x}\in\Omega}\,u(\bm{x})-\bm{x}^T\nabla v(\bm{x}^{(l)}).
\end{equation} 
If $\Omega$ is a convex set, then the above procedure solves a sequence of convex programs. Note that the same idea is used in the concave-convex procedure (CCCP) \citep{Yuille-03}. \par Suppose $u$ and $v$ are strictly convex, then a strict descent can be achieved in (\ref{Eq:sandwich}) unless $\bm{x}^{(l+1)}=\bm{x}^{(l)}$, i.e., if $\bm{x}^{(l+1)}\ne \bm{x}^{(l)}$, then
\begin{equation}\label{Eq:strictdescent}
f(\bm{x}^{(l+1)})< g(\bm{x}^{(l+1)},\bm{x}^{(l)})< g(\bm{x}^{(l)},\bm{x}^{(l)})=f(\bm{x}^{(l)}).
\end{equation}
The first strict inequality follows from (\ref{Eq:auxiliary-dc}), a strict inequality for strictly convex $v$. Since $u$ is strictly convex, $g$ is strictly convex and therefore $g(\bm{x}^{(l+1)},\bm{x}^{(l)})< g(\bm{x}^{(l)},\bm{x}^{(l)})$ unless $\bm{x}^{(l+1)}=\bm{x}^{(l)}$. This strictly monotonic descent property will be helpful to analyze the convergence of the sparse GEV algorithm that is presented in the following section.
\end{example}

\subsection{Sparse GEV algorithm} \label{Sec:algo}
Let us return to the approximate sparse GEV program in (\ref{Eq:tau}). Let
\begin{equation}\label{Eq:f}
f(\bm{x})=\tau\Vert\bm{x}\Vert^2_2+\rho_\varepsilon\sum^n_{i=1}\log(\varepsilon+|x_i|)-\bm{x}^T(\bm{A}+\tau\bm{I}_n)\bm{x},
\end{equation}
where $\tau\ge\max(0,-\lambda_{min}(\bm{A}))$ so that (\ref{Eq:tau}) can be written as $\min_{\bm{x}\in\Omega}f(\bm{x})$ and $\Omega=\{\bm{x}:\bm{x}^T\bm{Bx}\le 1\}$. The main idea in deriving the sparse GEV algorithm is in obtaining a majorization function, $g$ that satisfies (\ref{Eq:maj-step}) and then using it in (\ref{Eq:min-step}). The following result provides such a function $g$ for $f$ in (\ref{Eq:f}).
\begin{proposition}
The following function 
\begin{equation}\label{Eq:majorize-f}
g(\bm{x},\bm{y})=\tau\Vert\bm{x}\Vert^2_2-2\bm{x}^T(\bm{A}+\tau\bm{I}_n)\bm{y}+\bm{y}^T(\bm{A}+\tau\bm{I}_n)\bm{y}+\rho_\varepsilon\sum^n_{i=1}\log(\varepsilon+|y_i|)+\rho_\varepsilon\sum^n_{i=1}\frac{|x_i|-|y_i|}{|y_i|+\varepsilon},
\end{equation}
majorizes $f$ in (\ref{Eq:f}).
\end{proposition}
\begin{proof}
Consider the term $\log(\varepsilon+|x_i|)$ in $f$. Using the inequality $\log(z)\le z-1,\,\forall\,z\in\mathbb{R}_+$ with $z=\frac{|x_i|+\varepsilon}{|y_i|+\varepsilon}$, we have
\begin{equation}\label{Eq:log}
\log(\varepsilon+|x_i|)\le \log(\varepsilon+|y_i|)+\frac{|x_i|-|y_i|}{|y_i|+\varepsilon},\,\forall\,\bm{x},\bm{y}.
\end{equation}
On the other hand, since $\bm{A}+\tau\bm{I}_n\succeq 0$, by Example~\ref{Exm:dc} with $u(\bm{x})=\tau\Vert\bm{x}\Vert^2_2$ and $v(\bm{x})=\bm{x}^T(\bm{A}+\tau\bm{I}_n)\bm{x}$, we have
\begin{equation}\label{Eq:difftau}
\tau\Vert\bm{x}\Vert^2_2-\bm{x}^T(\bm{A}+\tau\bm{I}_n)\bm{x}\le \tau\Vert\bm{x}\Vert^2_2-\bm{y}^T(\bm{A}+\tau\bm{I}_n)\bm{y}-2(\bm{x}-\bm{y})^T(\bm{A}+\tau\bm{I}_n)\bm{y},\,\forall\,\bm{x},\bm{y}.
\end{equation}
From (\ref{Eq:log}) and (\ref{Eq:difftau}), it is easy to check that $g$ in (\ref{Eq:majorize-f}) majorizes $f$ over $\mathbb{R}^n\times\mathbb{R}^n$ and therefore over $\Omega\times\Omega$ where $\Omega=\{\bm{x}:\bm{x}^T\bm{Bx}\le 1\}$.
\end{proof}
\par By following the minimization step in (\ref{Eq:min-step}) with $g$ as in (\ref{Eq:majorize-f}), the \emph{sparse GEV algorithm} is obtained as 
\begin{equation}\tag{ALG}
\bm{x}^{(l+1)}=\arg\min_{\bm{x}}\left\{\tau\Vert\bm{x}\Vert^2_2-2\bm{x}^T(\bm{A}+\tau\bm{I}_n)\bm{x}^{(l)}+\rho_\varepsilon\sum^n_{i=1}\frac{|x_i|}{|x^{(l)}_i|+\varepsilon}\,:\,
\bm{x}^T\bm{Bx}\le 1\right\},
\end{equation}
which is a sequence of quadratically constrained quadratic programs (QCQPs) \citep{Boyd-06}.
It is clear that $\bm{x}^{(l+1)}$ is the unique optimal solution of (ALG) irrespective of whether $\tau$ is zero or not.\footnote{Suppose $\tau\ne 0$. The objective function in (ALG) is strictly convex in $\bm{x}$ and therefore $\bm{x}^{(l+1)}$ is the unique optimal solution. When $\tau=0$, the objective function is linear in $\bm{x}$ and the unique optimum lies on the boundary of the constraint set.} (ALG) can also be obtained by applying linear majorization (see Example~\ref{Exm:dc}) to (\ref{Eq:tau-1}). See Appendix B for details. 
\par Assuming $\tau\ne 0$ 
and defining $w^{(l)}_i:=\frac{1}{|x^{(l)}_i|+\varepsilon}$, $\bm{w}^{(l)}:=(w^{(l)}_1,\ldots,w^{(l)}_n)$ and $\bm{W}^{(l)}:=\bm{D}(\bm{w}^{(l)})$, a diagonal matrix with $\bm{w}^{(l)}$ as its principal diagonal, (ALG) reduces to
\begin{eqnarray}\label{Eq:weighted}
\bm{x}^{(l+1)}=\arg\min_{\bm{x}}&&\left\Vert\bm{x}-(\tau^{-1}\bm{A}+\bm{I}_n)\bm{x}^{(l)}\right\Vert^2_2+\frac{\rho_\varepsilon}{\tau}\left\Vert\bm{W}^{(l)}\bm{x}\right\Vert_1\nonumber\\
\text{s.t.}&&\bm{x}^T\bm{Bx}\le 1.
\end{eqnarray}
(\ref{Eq:weighted}) is very similar to LASSO \citep{Tibshirani-96} except for the \emph{weighted} $\ell_1$-penalty and the quadratic constraint. When $\bm{x}^{(0)}$ is chosen such that $\bm{x}^{(0)}=a\bm{1}$, 
then the first iteration of (\ref{Eq:weighted}) is a LASSO minimization problem except for the quadratic constraint. Let us analyze (\ref{Eq:weighted}) to get an intuitive interpretation.
\begin{itemize}
\item[(a)] $\rho_\varepsilon=\tilde{\rho}=0$: 
(\ref{Eq:weighted}) reduces to $\min\{\Vert\bm{x}-\bm{s}^{(l)}\Vert^2_2\,:\,\bm{x}^T\bm{Bx}\le 1\}$, where $\bm{s}^{(l)}=(\tau^{-1}\bm{A}+\bm{I}_n)\bm{x}^{(l)}$. So, if $\bm{s}^{(l)}\in\{\bm{x}:\bm{x}^T\bm{Bx}\le 1\}$, then $\bm{x}^{(l+1)}=\bm{s}^{(l)}$, else $\bm{x}^{(l+1)}=(\bm{I}_n+\mu^{(l+1)}\bm{B})^{-1}\bm{s}^{(l)}$, where $\mu^{(l+1)}$ satisfies $[\bm{s}^{(l)}]^T(\bm{I}_n+\mu^{(l+1)}\bm{B})^{-1}\bm{B}(\bm{I}_n+\mu^{(l+1)}\bm{B})^{-1}\bm{s}^{(l)}=1$. The first term in the objective of (\ref{Eq:weighted}) computes the best approximation to $\bm{s}^{(l)}$ in the $\ell_2$-norm so that the approximation lies in the ellipsoid $\bm{x}^T\bm{Bx}\le 1$. We show in Corollary~\ref{prop2} that the iterative algorithm in (\ref{Eq:weighted}) with $\rho_\varepsilon=0$
converges to the solution of (GEV-P) and therefore, the solution $\bm{x}$ is non-sparse.
\item[(b)] $\rho_\varepsilon=\tilde{\rho}=\infty$: In this case, (\ref{Eq:weighted}) reduces to $\min\{\Vert\bm{W}^{(l)}\bm{x}\Vert_1\,:\,\bm{x}^T\bm{Bx}\le 1\}$, which is a weighted $\ell_1$-norm minimization problem. Intuitively, it is clear that if $x^{(l)}_i$ is small, its weighting factor, $w^{(l)}_i=(|x^{(l)}_i|+\varepsilon)^{-1}$ in the next minimization step is large, which therefore pushes $x^{(l+1)}_i$ to be small. This way the small entries in $\bm{x}$ are generally pushed toward zero as far as the constraints on $\bm{x}$ allow, therefore yielding a sparse solution.
\end{itemize}
From the above discussion, it is clear that (\ref{Eq:weighted}) is a trade-off between the solution to the GEV problem and the solution to the weighted $\ell_1$-norm problem. 
From now on, we refer to (ALG) as the \emph{Sparse GEV algorithm}, which is detailed in Algorithm~\ref{alg1}.
\begin{algorithm}[t]
\caption{Sparse Generalized Eigenvalue Algorithm} \label{alg1}
\begin{algorithmic}[1]
\REQUIRE $\bm{A}\in\mathbb{S}^n$, $\bm{B}\succ 0$, $\varepsilon>0$ and $\tilde{\rho}>0$
\STATE 
$\text{Choose}\,\,\tau\ge\max(0,-\lambda_{min}(\bm{A}))$
\STATE
$\text{Choose}\,\,\bm{x}^{(0)}\in\{\bm{x}:\bm{x}^T\bm{Bx}\le 1\}$
\STATE
$\text{Set}\,\,\rho_\varepsilon=\frac{\tilde{\rho}}{\log(1+\varepsilon^{-1})}$
\IF{$\tau=0$} 
\REPEAT 
\STATE $w^{(l)}_i=(|x^{(l)}_i|+\varepsilon)^{-1}$
\STATE $\bm{W}^{(l)}=\bm{D}(\bm{w}^{(l)})$
\STATE\begin{eqnarray}\label{Eq:weighted-algo}
\bm{x}^{(l+1)}=\arg\max_{\bm{x}}&&\bm{x}^T\bm{Ax}^{(l)}-\frac{\rho_\varepsilon}{2}\left\Vert\bm{W}^{(l)}\bm{x}\right\Vert_1\nonumber\\
\text{s.t.}&&\bm{x}^T\bm{Bx}\le 1.
\end{eqnarray}
\UNTIL $\text{convergence}$
\ELSE 
\REPEAT 
\STATE $w^{(l)}_i=(|x^{(l)}_i|+\varepsilon)^{-1}$
\STATE $\bm{W}^{(l)}=\bm{D}(\bm{w}^{(l)})$
\STATE\begin{eqnarray}\label{Eq:weighted-nonzero-algo}
\bm{x}^{(l+1)}=\arg\min_{\bm{x}}&&\left\Vert\bm{x}-(\tau^{-1}\bm{A}+\bm{I}_n)\bm{x}^{(l)}\right\Vert^2_2+\frac{\rho_\varepsilon}{\tau}\left\Vert\bm{W}^{(l)}\bm{x}\right\Vert_1\nonumber\\
\text{s.t.}&&\bm{x}^T\bm{Bx}\le 1.
\end{eqnarray}
\UNTIL $\text{convergence}$
\ENDIF
\STATE \textbf{return}
$\bm{x}^{(l)}$
\end{algorithmic}
\end{algorithm}

To run Algorithm~\ref{alg1}, $\tilde{\rho}$, $\tau$ and $\varepsilon$ need to be chosen. In a supervised learning setup like FDA,
$\tilde{\rho}$ can be chosen by cross-validation whereas, in an
unsupervised setup like PCA/CCA, Algorithm~\ref{alg1} has to be
solved for various $\tilde{\rho}$ and the solution with desired
cardinality is selected. 
Since $\tilde{\rho}$ is a free parameter, $\tau$ and $\varepsilon$ can be set to any value (that satisfies the constraints in Algorithm~\ref{alg1}) and $\tilde{\rho}$ can be tuned to obtain the desired sparsity as mentioned above. However, it has to be noted that for a fixed value of $\tilde{\rho}$, increasing $\tau$ or $\varepsilon$ reduces sparsity.\footnote{Increasing $\varepsilon$ increases the approximation error between $\Vert\bm{x}\Vert_0$ and $\sum^n_{i=1}\frac{\log(1+|x_i|\varepsilon^{-1})}{\log(1+\varepsilon^{-1})}$ and therefore reduces sparsity. From (\ref{Eq:weighted}), it is clear that increasing $\tau$ reduces the weight on the term $\Vert\bm{W}^{(l)}\bm{x}\Vert_1$, which means more importance is given to reducing the approximation error, $\Vert\bm{x}-(\tau^{-1}\bm{A}+\bm{I}_n)\bm{x}^{(l)}\Vert^2_2$, leading to a less sparse solution.} So, in practice $\tau$ is chosen to be $\max(0,-\lambda_{min}(\bm{A}))$, $\varepsilon$ to be \emph{close} to zero and $\tilde{\rho}$ is set by searching for a value that provides the desired sparsity.
\par Suppose that Algorithm~\ref{alg1} outputs a solution, $\bm{x}^*$ such that $\Vert\bm{x}^*\Vert_0=k$. 
Can we say that $\bm{x}^*$ is the optimal solution of (SGEV-P) among all $\bm{x}$ 
with cardinality $k$?
The following proposition provides a condition to check for the non-optimality of $\bm{x}^*$. In addition, it also presents a post-processing step (called \emph{variational renormalization}) that \emph{improves} the performance of Algorithm~\ref{alg1}. See \citet[Proposition 2]{Moghaddam-06a} for a similar result in the case of $\bm{A}\succeq 0$ and $\bm{B}=\bm{I}_n$.
\begin{proposition}
Suppose Algorithm~\ref{alg1} converges to a solution $\bm{x}^*$ such that $\Vert\bm{x}^*\Vert_0= k$. Let $\bm{z}$ be the sub-vector of $\bm{x}^*$ (obtained by removing the zero entries of $\bm{x}^*$) and $\bm{u}_k=\arg\max\{\bm{x}^T\bm{A}_k\bm{x}:\bm{x}^T\bm{B}_k\bm{x}=1\}$, where $\bm{A}_k$ and $\bm{B}_k$ are submatrices of $\bm{A}$ and $\bm{B}$ defined by the same non-zero indices of $\bm{x}^*$. 
If $\bm{z}\ne \bm{u}_k$, then $\bm{x}^*$ is not the optimal solution of (SGEV-P) among all $\bm{x}$ with the same sparsity pattern as $\bm{x}^*$ (and therefore, is not the optimal solution of (SGEV-P) among all $\bm{x}$ with $\Vert\bm{x}\Vert_0=k$).
Nevertheless, by replacing the non-zero entries of $\bm{x}^*$ with those of $\bm{u}_k$, 
the value of the objective function in (SGEV-P) increases from $[\bm{x}^*]^T\bm{Ax}^*$ to $\lambda(\bm{A}_k,\bm{B}_k)$, its optimal value among all $\bm{x}$ with the same sparsity pattern as $\bm{x}^*$.
\end{proposition}
\begin{proof}
Assume that $\bm{x}^*$, the solution output by Algorithm~\ref{alg1}, is the optimal solution of (SGEV-P).
Define $\bm{v}$ such that $v_i=\mathds{1}_{\{|x^*_i|\ne 0\}}$. Since $\bm{x}^*$ is the optimal solution of (SGEV-P), we have $\bm{x}^*=\arg\max\{\bm{y}^T\bm{D}(\bm{v})\bm{A}\bm{D}(\bm{v})\bm{y}:\bm{y}^T\bm{D}(\bm{v})\bm{B}\bm{D}(\bm{v})\bm{y}=1\}$, which is equivalent to $\bm{z}=\arg\max\{\bm{w}^T\bm{A}_k\bm{w}:\bm{w}^T\bm{B}_k\bm{w}=1\}=\bm{u}_k$ and the result follows. Note that $\lambda(\bm{A}_k,\bm{B}_k)$ is the optimal value of (SGEV-P) among all $\bm{x}$ with the same sparsity pattern as $\bm{x}^*$. Therefore, if $\bm{z}=\bm{u}_k$, then $[\bm{x}^*]^T\bm{Ax}^*=\lambda(\bm{A}_k,\bm{B}_k)$.
\end{proof}
The variational renormalization suggests that given a solution (in our case, $\bm{x}^*$ at the termination of Algorithm~\ref{alg1}), it is almost certainly better to discard the loadings, keep only the sparsity pattern and solve the smaller unconstrained subproblem to obtain the final loadings, given the sparsity pattern. This procedure surely improves any continuous algorithm's performance. 

\par In Algorithm~\ref{alg1}, we mention that the iterative scheme is continued until convergence. 
What does convergence mean here? Does the algorithm really converge? If it converges, what does it converge to? Does it converge to an optimal solution of (SGEV-A)? To address these questions, in the following section, we provide the convergence analysis of Algorithm~\ref{alg1} using tools from global convergence theory~\citep{Zangwill-69}.

\subsection{Convergence analysis}\label{Sec:convergence}
For an iterative procedure like Algorithm~\ref{alg1} to be useful, it must converge to point solutions from all or at least a significant number of initialization states and not exhibit other nonlinear system behaviors, such as divergence or oscillation. \emph{Global convergence theory of iterative algorithms} \citep{Zangwill-69} can be used to investigate this behavior. We mention up front that this \emph{does not} deal with proving convergence to a global optimum. Using this theory, recently, \citet{Sriperumbudur-09a} analyzed the convergence behavior of the iterative algorithm in (\ref{Eq:example}) and showed that under certain conditions on $u$ and $v$, the algorithm in (\ref{Eq:example}) is globally convergent, i.e., for any random initialization, $\bm{x}^{(0)}$, the 
sequence of iterates $\{\bm{x}^{(l)}\}^\infty_{l=0}$ converges to some stationary point\footnote{$\bm{x}_\ast$ is said to be a stationary point of a constrained optimization problem if it satisfies the corresponding Karush-Kuhn-Tucker (KKT) conditions~\citep[Section 13.3]{Bonnans-06}. Assuming constraint qualification, KKT conditions are necessary for the local optimality of $\bm{x}_\ast$.} of the d.c. program, $\min_{\bm{x}\in\Omega}(u(\bm{x})-v(\bm{x}))$. Since (ALG) can be obtained by applying linear majorization to (\ref{Eq:tau-1}), as shown in Appendix B, the convergence analysis of (ALG) can be carried out (see Theorem~\ref{Thm:alg-1}) by invoking the results in \cite{Sriperumbudur-09a}. We show in Theorem~\ref{Thm:alg-1} that Algorithm~\ref{alg1} is globally convergent. 
In Corollary~\ref{prop2}, we then show that Algorithm~\ref{alg1} converges to the solution of the GEV problem in (GEV-P) when $\tilde{\rho}=0$. In the following, we introduce some notation and terminology and then proceed with the derivation of the above mentioned results. 
\par The convergence analysis of an iterative procedure like Algorithm~\ref{alg1} uses the notion of a \emph{set-valued mapping}, or \emph{point-to-set mapping}, which is central to the theory of global convergence. A point-to-set map $\Psi$ from a set $X$ into a set $Y$ is defined as $\Psi:X\rightarrow\mathscr{P}(Y)$, which assigns a subset of $Y$ with each point of $X$, where $\mathscr{P}(Y)$ denotes the power set of $Y$. 
$\Psi$ is said to be \emph{uniformly compact} on $X$ if there exists a compact set $H$ independent of $x$ such that $\Psi(x)\subset H$ for all $x\in X$. Note that if $X$ is compact, then $\Psi$ is uniformly compact on $X$. A \emph{fixed point} of the map $\Psi:X\rightarrow \mathscr{P}(X)$ is a point $x$ for which $\{x\}=\Psi(x)$. 
\par Many iterative algorithms in mathematical programming can be described using the notion of point-to-set maps. Let $X$ be a set and $x_0\in X$ a given point. Then an \emph{algorithm}, $\EuScript{A}$, with initial point $x_0$ is a point-to-set map $\EuScript{A}:X\rightarrow \mathscr{P}(X)$ which generates a sequence $\{x_k\}^\infty_{k=1}$ via the rule $x_{k+1}\in\EuScript{A}(x_k),\,k=0,1,\ldots$. $\EuScript{A}$ is said to be \emph{globally convergent} if \emph{for any chosen initial point} $x_0$, the sequence $\{x_k\}^\infty_{k=0}$ generated by $x_{k+1}\in\EuScript{A}(x_k)$ (or a subsequence) converges to a point for which a necessary condition of optimality holds: the Karush-Kuhn-Tucker (KKT) conditions in the case of constrained optimization and stationarity in the case of unconstrained optimization. The property of global convergence expresses, in a sense, the certainty that the algorithm works. It is very important to stress the fact that it does not imply (contrary to what the term might suggest) convergence to a global optimum for all initial points $x_0$. 
\par We now state the convergence result for (\ref{Eq:example}) by \citet{Sriperumbudur-09a}, using which we provide the convergence result for (ALG).
\begin{theorem}[\cite{Sriperumbudur-09a}]\label{Thm:globalCCCP-II}
Consider the program,
\begin{equation}\tag{DC}
\min\{u(x)-v(x)\,:\,x\in\Omega\},
\end{equation}
where $\Omega=\{x:c_i(x)\le 0,\,i\in[m],\,d_j(x)=0,\,j\in[p]\}$ and $[m]:=\{1,\ldots,m\}$. Let $u$ and $v$ be strictly convex, differentiable functions defined on $\bb{R}^n$. Also assume $\nabla v$ is continuous. Let $\{c_i\}$ be differentiable convex functions and $\{d_j\}$ be affine functions on $\bb{R}^n$. Suppose (DC) is solved iteratively as $x^{(l+1)}\in\EuScript{A}_{dc}(x^{(l)})$, where $\EuScript{A}_{dc}$ is the point-to-set map defined as
\begin{equation}\tag{DC-ALG}
\EuScript{A}_{dc}(y)=\arg\min_{x\in\Omega}\,\,u(x)-x^Tv(y).
\end{equation}
Let $\{x^{(l)}\}^\infty_{l=0}$ be any sequence generated by $\EuScript{A}_{dc}$ defined by (DC-ALG). Suppose $\EuScript{A}_{dc}$ is uniformly compact on $\Omega$ and $\EuScript{A}_{dc}(x)$ is nonempty for any $x\in\Omega$. Then, assuming suitable constraint qualification, all the limit points of $\{x^{(l)}\}^\infty_{l=0}$ are stationary points of the d.c. program in (DC), $u(x^{(l)})-v(x^{(l)})\rightarrow u(x_\ast)-v(x_\ast)=:f^\ast$ as $l\rightarrow\infty$, for some stationary point $x_\ast$, $\Vert x^{(l+1)}-x^{(l)}\Vert\rightarrow 0$, and either $\{x^{(l)}\}^\infty_{l=0}$ converges or the set of limit points of $\{x^{(l)}\}^\infty_{l=0}$ is a connected and compact subset of $\mathscr{S}(f^\ast)$, where $\mathscr{S}(a):=\{x\in\mathscr{S}:u(x)-v(x)=a\}$ and $\mathscr{S}$ is the set of stationary points of (DC). If $\mathscr{S}(f^\ast)$ is finite, then any sequence $\{x^{(l)}\}^\infty_{l=0}$ generated by $\EuScript{A}_{dc}$ converges to some $x_\ast$ in $\mathscr{S}(f^\ast)$.
\end{theorem}
The following global convergence theorem for (ALG) is obtained by simply invoking Theorem~\ref{Thm:globalCCCP-II} with $\Omega=\{(\bm{x},\bm{y}):\bm{x}^T\bm{Bx}\le 1,\,-\bm{y}\preceq \bm{x}\preceq\bm{y}\}$, $u(\bm{x})=\tau\Vert\bm{x}\Vert^2_2$ and $v(\bm{x},\bm{y})=\bm{x}^T(\bm{A}+\tau\bm{I}_n)\bm{x}-\rho_\varepsilon\sum^n_{i=1}\log(y_i+\varepsilon)$.
\begin{theorem}[Global convergence of sparse GEV algorithm]\label{Thm:alg-1}
Let $\{\bm{x}^{(l)}\}^\infty_{l=0}$ be any sequence generated by the sparse GEV algorithm in Algorithm~\ref{alg1}. Then, all the limit points of $\{\bm{x}^{(l)}\}^\infty_{l=0}$ are stationary points of the program in (SGEV-A),
\begin{equation}
\rho_\varepsilon\sum^n_{i=1}\log(\varepsilon+|x^{(l)}_i|)-[\bm{x}^{(l)}]^T\bm{A}\bm{x}^{(l)}\rightarrow
\rho_\varepsilon\sum^n_{i=1}\log(\varepsilon+|x^\ast_i|)-[\bm{x}^\ast]^T\bm{A}\bm{x}^\ast:=L^\ast,\end{equation}  for some stationary point $\bm{x}^\ast$, $\Vert\bm{x}^{(l+1)}-\bm{x}^{(l)}\Vert\rightarrow 0$, and either $\{\bm{x}^{(l)}\}^\infty_{l=0}$ converges or the set of limit points of $\{\bm{x}^{(l)}\}^\infty_{l=0}$ is a connected and compact subset of $\mathscr{S}(L^\ast)$, where $\mathscr{S}(a):=\{\bm{x}\in\mathscr{S}:\bm{x}^T\bm{Ax}-\rho_\varepsilon\sum^n_{i=1}\log(\varepsilon+|x_i|)=-a\}$ and $\mathscr{S}$ is the set of stationary points of (SGEV-A). If $\mathscr{S}(L^\ast)$ is finite, then any sequence $\{\bm{x}^{(l)}\}^\infty_{l=0}$ generated by Algorithm~\ref{alg1} converges to some $\bm{x}^\ast$ in $\mathscr{S}(L^\ast)$.
\end{theorem}
\begin{proof}
As noted before, (ALG) can be obtained by applying linear majorization to (\ref{Eq:tau-1}), which is equivalent to (SGEV-A), with $\Omega=\{(\bm{x},\bm{y}):\bm{x}^T\bm{Bx}\le 1,\,-\bm{y}\preceq \bm{x}\preceq\bm{y}\}$, $u(\bm{x})=\tau\Vert\bm{x}\Vert^2_2$ and $v(\bm{x},\bm{y})=\bm{x}^T(\bm{A}+\tau\bm{I}_n)\bm{x}-\rho_\varepsilon\sum^n_{i=1}\log(y_i+\varepsilon)$. It is easy to check that $u$ and $v$ satisfy the conditions of Theorem~\ref{Thm:globalCCCP-II}. Since Algorithm~\ref{alg1} and (ALG) are equivalent, let $\EuScript{A}$ correspond to the point-to-set map in (ALG). Clearly $\{\bm{x}:\bm{x}^T\bm{Bx}\le 1\}$ is compact and therefore $\EuScript{A}$ is uniformly compact. By the Weierstrass theorem\footnote{The Weierstrass theorem states: If $f$ is a real continuous function on a compact set $K\subset\mathbb{R}^n$, then the problem $\min\{f(x):x\in K\}$ has an optimal solution $x^*\in K$.}~\citep{Minoux-86}, it is clear that $\EuScript{A}(\bm{x})$ is nonempty for any $\bm{x}\in\{\bm{x}:\bm{x}^T\bm{Bx}\le 1\}$. The result therefore follows from Theorem~\ref{Thm:globalCCCP-II}.
\end{proof}
Having considered the convergence of Algorithm~\ref{alg1}, we now consider the convergence of special cases of Algorithm~\ref{alg1}. Note that $\tilde{\rho}=0$ implies $\rho_\varepsilon=0$. Using this in (SGEV-A) yields the GEV problem in (GEV-P). Since Algorithm~\ref{alg1} is derived based on (SGEV-A), it would be of interest to know whether the sequence $\{\bm{x}^{(l)}\}^\infty_{l=0}$ generated by Algorithm~\ref{alg1} for $\tilde{\rho}=0$ converges to the solution of (GEV-P). The following corollary answers this and shows that when $\tilde{\rho}=0$, the solution of the sparse GEV algorithm (Algorithm~\ref{alg1}) matches with that of the GEV problem in (GEV-P).
\begin{corollary}\label{prop2}
Let $\tilde{\rho}=0$ and $\lambda_{max}(\bm{A},\bm{B})>0$.\footnote{If $\lambda_{max}(\bm{A},\bm{B})>0$, then $\max\{\bm{x}^T\bm{Ax}:\bm{x}^T\bm{Bx}\le 1\}=\max\{\bm{x}^T\bm{Ax}:\bm{x}^T\bm{Bx}= 1\}=\lambda_{max}(\bm{A},\bm{B})$.} 
Then, any sequence $\{\bm{x}^{(l)}\}^\infty_{l=0}$ generated by Algorithm~\ref{alg1} converges to some $\bm{x}^\ast$ such that $\lambda_{max}(\bm{A},\bm{B})=[\bm{x}^\ast]^T\bm{Ax}^\ast$ and $[\bm{x}^\ast]^T\bm{Bx}^\ast=1$.
\end{corollary}
\begin{proof} 
The stationary points of (SGEV-A) with $\rho_\varepsilon=\tilde{\rho}=0$ are the generalized eigenvectors of $(\bm{A},\bm{B})$. Therefore, the set $\mathscr{S}$ as defined in Theorem~\ref{Thm:alg-1} is finite and any sequence $\{\bm{x}^{(l)}\}^\infty_{l=0}$ generated by Algorithm~\ref{alg1} converges to some $\bm{x}^\ast$ in $\mathscr{S}(L^\ast)$ where $L^\ast=-[\bm{x}^\ast]^T\bm{Ax}^\ast$. We need to show that $L^\ast=-\lambda_{max}(\bm{A},\bm{B})$. Note that $\bm{x}^\ast$ is a fixed point of Algorithm~\ref{alg1}. Consider (ALG) which is equivalent to Algorithm~\ref{alg1}. With $\rho_\varepsilon=0$, solving the Lagrangian yields $\bm{x}^{(l+1)}=(\mu^{(l+1)}\bm{B}+\tau\bm{I}_n)^{-1}(\bm{A}+\tau\bm{I}_n)\bm{x}^{(l)}$, where $\mu^{(l+1)}\ge 0$ is the Lagrangian dual variable for the constraint $[\bm{x}^{(l+1)}]^T\bm{Bx}^{(l+1)}\le 1$. At the fixed point, $\bm{x}^\ast$, we have $(\mu^\ast\bm{B}+\tau\bm{I}_n)\bm{x}^\ast=(\bm{A}+\tau\bm{I}_n)\bm{x}^\ast$ which implies 
\begin{equation}\label{Eq:proof}
\bm{Ax}^\ast=\mu^\ast\bm{Bx}^\ast.
\end{equation} 
Multiplying both sides of (\ref{Eq:proof}) by $[\bm{x}^\ast]^T$, we have \begin{eqnarray}\label{Eq:slackness}
[\bm{x}^\ast]^T\bm{Ax}^\ast&=&\mu^\ast[\bm{x}^\ast]^T\bm{Bx}^\ast=\mu^\ast([\bm{x}^\ast]^T\bm{Bx}^\ast-1)+\mu^\ast\nonumber\\
&=&\mu^\ast,
\end{eqnarray} where we have invoked the complementary slackness condition, $\mu^\ast([\bm{x}^\ast]^T\bm{Bx}^\ast-1)=0$. The optimum value of (ALG) at the fixed point is given by $\psi^\ast:=-2[\bm{x}^\ast]^T\bm{Ax}^\ast-\tau\Vert\bm{x}^\ast\Vert^2_2$, which by (\ref{Eq:slackness}) reduces to $\psi^\ast=-2\mu^\ast-\tau\Vert\bm{x}^\ast\Vert^2_2$. It is easy to see that making $\mu^\ast>0$, and therefore $[\bm{x}^\ast]^T\bm{Bx}^\ast=1$ minimizes $\psi^\ast$ instead of choosing $\mu^\ast=0$ and $[\bm{x}^\ast]^T\bm{Bx}^\ast<1$. Since $\psi^\ast$ is minimized by choosing the maximum $\mu^\ast$ that satisfies (\ref{Eq:proof}), $(\mu^\ast,\bm{x}^\ast)$ is indeed the eigen pair that satisfies the GEV problem in (GEV-P).
\end{proof}
In addition to $\tilde{\rho}=0$, suppose $\bm{A}\succeq 0$ and $\tau=0$. Then, the following result shows that a simple iterative algorithm can be obtained to compute the generalized eigenvector associated with $\lambda_{max}(\bm{A},\bm{B})$.
\begin{corollary}\label{pro:iterative}
Let $\bm{A}\succeq 0$, $\tau=0$ and $\tilde{\rho}=0$. Then, any sequence $\{\bm{x}^{(l)}\}^\infty_{l=0}$ generated by the following algorithm
\begin{equation}\label{Eq:power-generalize}
\bm{x}^{(l+1)}=\frac{\bm{B}^{-1}\bm{A}\bm{x}^{(l)}}{\sqrt{[\bm{x}^{(l)}]^T\bm{AB}^{-1}\bm{Ax}^{(l)}}}
\end{equation}
converges to some $\bm{x}^\ast$ such that $\lambda_{max}(\bm{A},\bm{B})=[\bm{x}^\ast]^T\bm{Ax}^\ast$ and $[\bm{x}^\ast]^T\bm{Bx}^\ast=1$.
\end{corollary}
\begin{proof}
Consider (ALG) with $\tau=0$ and $\rho_\varepsilon=\tilde{\rho}=0$. Since the objective is linear in $\bm{x}$, the minimum occurs at the boundary of the constraint set, i.e., $\{\bm{x}:\bm{x}^T\bm{Bx}=1\}$. Solving the Lagrangian, we get (\ref{Eq:power-generalize}). The result therefore follows from Corollary~\ref{prop2} which holds for any $\tau\ge 0$. 
\vspace{-3mm}
\end{proof}
\par So far, we have proposed a sparse GEV algorithm and proved its global convergence behavior. In the following sections (Sections~\ref{sec:pca}-\ref{Sec:sparseFDA}), we consider specific instances of the sparse GEV problem and use the proposed algorithm (Algorithm~\ref{alg1}) to address them.

\section{Sparse Principal Component Analysis}\label{sec:pca}
In this section, we consider sparse PCA as a special case of the sparse GEV algorithm that we presented in Section~\ref{Sec:algo}. Based on the sparse GEV algorithm in Algorithm~\ref{alg1}, we propose a sparse PCA algorithm, called DC-PCA, with $\bm{A}\succeq 0$ being the covariance matrix, $\bm{B}=\bm{I}_n$ and $\tau=0$. 
In this setting, (ALG) reduces to a very simple iterative rule, which is proved in Appendix C:
\begin{equation}\tag{ALG-S}
x^{(l+1)}_i=\frac{\left[\left|(\bm{Ax}^{(l)})_i\right|-\frac{\rho_\varepsilon}{2}w^{(l)}_i\right]_+\text{sign}((\bm{Ax}^{(l)})_i)}{\sqrt{\sum^n_{i=1}\left[\left|(\bm{Ax}^{(l)})_i\right|-\frac{\rho_\varepsilon}{2}w^{(l)}_i\right]^2_+}},\,\forall\,i,
\end{equation}
where $[a]_+:=\max(0,a)$. The corresponding sparse PCA algorithm (DC-PCA) is shown in Algorithm~\ref{alg2}.
\begin{algorithm}[t]
\caption{Sparse PCA algorithm (DC-PCA)} \label{alg2}
\begin{algorithmic}[1]
\REQUIRE $\bm{A}\succeq 0$, $\varepsilon>0$ and $\tilde{\rho}>0$
\STATE
$\text{Choose}\,\,\bm{x}^{(0)}\in\{\bm{x}:\bm{x}^T\bm{x}\le 1\}$
\STATE
$\text{Set}\,\,\rho_\varepsilon=\frac{\tilde{\rho}}{\log(1+\varepsilon^{-1})}$
\REPEAT
\STATE $w^{(l)}_i=(|x^{(l)}_i|+\varepsilon)^{-1}$
\STATE 
\begin{equation}
x^{(l+1)}_i=\frac{\left[\left|(\bm{Ax}^{(l)})_i\right|-\frac{\rho_\varepsilon}{2}w^{(l)}_i\right]_+\text{sign}((\bm{Ax}^{(l)})_i)}{\sqrt{\sum^n_{i=1}\left[\left|(\bm{Ax}^{(l)})_i\right|-\frac{\rho_\varepsilon}{2}w^{(l)}_i\right]^2_+}},\,\forall\,i
\end{equation}
\UNTIL $\text{convergence}$
\STATE \textbf{return}
$\bm{x}^{(l)}$
\end{algorithmic}
\end{algorithm}
Note that the computation of $\bm{x}^{(l+1)}$, from $\bm{x}^{(l)}$, involves computing $\bm{Ax}^{(l)}$, which has a complexity of $O(n^2)$. Therefore, the DC-PCA algorithm has a per iteration complexity of $O(n^2)$. Since Algorithm~\ref{alg1} exhibits the global convergence behavior and DC-PCA is a special case of Algorithm~\ref{alg1}, it follows that DC-PCA exhibits the global convergence property.
\par Suppose $\rho_\varepsilon=\tilde{\rho}=0$. Then, 
with $\bm{A}\succeq 0$ and $\bm{B}=\bm{I}_n$, (SGEV-R) reduces to:
\begin{equation}\tag{EV-P}
\max\{\bm{x}^T\bm{Ax}:\bm{x}^T\bm{x}=1\},
\end{equation}
i.e., a standard eigenvalue problem.
Therefore, it is of interest to know whether Algorithm~\ref{alg2} provides a solution of (EV-P) when $\tilde{\rho}=0$. It follows from Corollary~\ref{pro:iterative} that DC-PCA converges to a solution of (EV-P) when $\tilde{\rho}=0$. In addition, the following result shows that DC-PCA reduces to the power method for computing $\lambda_{max}(\bm{A})$ when $\tilde{\rho}=0$.
\begin{proposition}[Power method]\label{cor:power} Suppose $\tilde{\rho}=0$. Then, Algorithm~\ref{alg2} is the power method for computing $\lambda_{max}(\bm{A})$.
\end{proposition}
\begin{proof}
Setting $\tilde{\rho}=0$ in Algorithm~\ref{alg2} yields 
\begin{equation}\label{Eq:power}
\bm{x}^{(l+1)}=\frac{\bm{Ax}^{(l)}}{\Vert\bm{Ax}^{(l)}\Vert_2},
\end{equation}
which is the power iteration for the computation of $\lambda_{max}(\bm{A})$. 
\vspace{-4mm}
\end{proof}
\par Before proceeding further, we briefly discuss the prior work on sparse PCA algorithms. The earliest attempts at ``sparsifying" PCA consisted of simple axis rotations and component thresholding~\citep{Cadima-95} for subset selection, often based on the identification of principal variables \citep{McCabe-84}. The first true computational technique, called SCoTLASS \citep{Jolliffe-03}, provided an optimization framework using LASSO~\citep{Tibshirani-96} by enforcing a sparsity constraint on the PCA solution by bounding its $\ell_1$-norm, leading to a non-convex procedure. \citet{Zou-06} proposed a $\ell_1$-penalized regression algorithm for PCA (called SPCA) using an \emph{elastic net} \citep{Zou-05} and solved it very efficiently using least angle regression \citep{Efron-04}. Subsequently, \citet{Aspremont-07} proposed a convex relaxation to the non-convex cardinality constraint for PCA (called DSPCA) leading to a SDP with a complexity of $O(n^4\sqrt{\log n})$. Although this method shows performance comparable to SPCA on a small-scale benchmark data set, it is not scalable for high-dimensional data sets, even possibly with Nesterov's first-order method \citep{Nesterov-05}. \citet{Moghaddam-06a} proposed a combinatorial optimization algorithm (called GSPCA) using greedy search and branch-and-bound methods to solve the sparse PCA problem, leading to a total complexity of $O(n^4)$ for a full set of solutions (one for each target sparsity between $1$ and $n$). \citet{Aspremont-08} formulated a new SDP relaxation to the sparse PCA problem and derived a more efficient greedy algorithm (compared to GSPCA) for computing a full set of solutions at a total numerical complexity of $O(n^3)$, which is based on the convexity of the largest eigenvalue of a symmetric matrix.
Recently, \citet{Journee-08} proposed a simple, iterative and very efficient sparse PCA algorithm ($\text{GPower}_{\ell_0}$) with a per iteration complexity of $O(n^2)$, which is based on the idea of linear majorization. They showed that it performs similar to many of the above mentioned algorithms. Therefore, to our knowledge, $\text{GPower}_{\ell_0}$ is the state-of-the-art.
\par In the following, we discuss how DC-PCA relates to SCoTLASS~\citep{Jolliffe-03}, SPCA~\citep{Zou-06} and GPower$_{\ell_0}$~\citep{Journee-08}. We then present experiments to empirically compare different approaches.

\subsection{Comparison to SCoTLASS}~\label{subsec:scotlass}
As mentioned before, the SCoTLASS program is obtained by approximating $\Vert\bm{x}\Vert_0$ with $\Vert\bm{x}\Vert_1$ in (SGEV-P) and is given by 
\begin{eqnarray}
\max_{\bm{x}}&&\bm{x}^T\bm{Ax}\nonumber\\
\text{s.t.}&&\Vert\bm{x}\Vert^2_2=1,\,\Vert\bm{x}\Vert_1\le k,
\end{eqnarray}
where $\bm{A}\succeq 0$. Let us consider the regularized version of the above program, given by
\begin{eqnarray}\label{Eq:scotlass-regularized}
\max_{\bm{x}}&& \bm{x}^T\bm{Ax}-\rho\Vert\bm{x}\Vert_1\nonumber\\
\text{s.t.}&&\Vert\bm{x}\Vert^2_2\le 1.
\end{eqnarray}
It is clear that (\ref{Eq:scotlass-regularized}) is not a canonical convex program because of the convex maximization. Applying the MM algorithm to (\ref{Eq:scotlass-regularized}), we obtain the following iterative algorithm:
\begin{eqnarray}\label{Eq:scot}
\bm{x}^{(l+1)}=\arg\max_{\bm{x}}&&
\bm{x}^T\bm{A}\bm{x}^{(l)}-\frac{\rho}{2}\Vert\bm{x}\Vert_1\nonumber\\
\text{s.t.}&&\Vert\bm{x}\Vert^2_2\le 1.
\end{eqnarray}
Using an approach as in Appendix C, (\ref{Eq:scot}) can be solved in closed form as
\begin{equation}\label{Eq:scotl}
x^{(l+1)}_i=\frac{\left[\left|(\bm{Ax}^{(l)})_i\right|-\frac{\rho}{2}\right]_+\text{sign}((\bm{Ax}^{(l)})_i)}{\sqrt{\sum^n_{i=1}\left[\left|(\bm{Ax}^{(l)})_i\right|-\frac{\rho}{2}\right]^2_+}},\,\forall\,i.
\end{equation}
Now, let us compare (\ref{Eq:scotl}) with the DC-PCA iteration in Algorithm~\ref{alg2}. For ScoTLASS, $x^{(l+1)}_i=0$ if $|(\bm{Ax}^{(l)})_i|\le\frac{\rho}{2}$, whereas for DC-PCA, $x^{(l+1)}_i=0$ if $|(\bm{Ax}^{(l)})_i|\le\frac{\rho_\varepsilon}{2}w^{(l)}_i$. This means that if $x^{(l)}_i=0$ for some $l$, 
then DC-PCA ensures that
$x^{(m)}_i=0,\,\forall\, m>l$ which is not guaranteed for SCoTLASS. Therefore, DC-PCA ensures faster convergence of an irrelevant feature to zero than
SCoTLASS, thus providing better sparsity. This is not surprising as a better approximation to the cardinality constraint is used in DC-PCA. It can be shown that when $\rho=0$, like DC-PCA, SCoTLASS also reduces to the \emph{power iteration algorithm} in (\ref{Eq:power}).

\subsection{Comparison to SPCA}
Let $\bm{Q}$ be a $r\times n$ matrix, where $r$ and $n$ are the number of observations and the number of variables respectively, with the column means being zero. Suppose $\bm{Q}$ has an SVD given by $\bm{Q}=\bm{U\Lambda V}^T$, where $\bm{U}$ contains the principal components of unit length and the columns of $\bm{V}$ are the corresponding loadings of the principal components. Let $\bm{y}_i=[\bm{U\Lambda}]_i,\,\forall\,i$. \citet[Theorem 1]{Zou-06} posed PCA as a regression problem and showed that $[\bm{V}]_i=\bm{x}^\star/\Vert\bm{x}^\star\Vert_2$, where
\begin{equation}\label{Eq:ridge}
\bm{x}^\star=\arg\min_{\bm{x}}\,\Vert\bm{y}_i-\bm{Qx}\Vert^2_2+\lambda\Vert\bm{x}\Vert^2_2,
\end{equation}
where $\lambda>0$. This is equivalent to solving for an eigenvector of $\bm{Q}^T\bm{Q}=:\bm{A}$. Therefore, solving for the eigenvectors of a positive semidefinite matrix is posed as a ridge regression problem in (\ref{Eq:ridge}). To solve for sparse eigenvectors, \citet{Zou-06} introduced an $\ell_1$-penalty term in (\ref{Eq:ridge}) resulting in the following \emph{elastic net} called SPCA, 
\begin{equation}\label{Eq:EN-SPCA}
\bm{x}^\prime=\arg\min_{\bm{x}}\,\Vert\bm{y}_i-\bm{Qx}\Vert^2_2+\lambda\Vert\bm{x}\Vert^2_2+\lambda_1\Vert\bm{x}\Vert_1,
\end{equation}
where $\lambda_1>0$. This problem can be interpreted in a Bayesian setting as follows: given the likelihood on $\bm{y}_i$, $\bm{y}_i|\bm{x},\sigma^2\sim\mathcal{G}(\bm{Qx},\sigma^2\bm{I})$, which is a circular normal random variable with mean $\bm{Qx}$ (conditioned on $\bm{x}$), and a prior distribution on $\bm{x}$,
$\bm{x}|\beta^2,\gamma\sim\mathcal{G}(\bm{0},\beta^2\bm{I})\prod_i\exp(-\gamma|x_i|)$, which is the product of a circular Gaussian and a product of Laplacian densities, compute the maximum a posteriori (MAP) estimate of $\bm{x}$. It is easy to see that the penalization parameters $\lambda$ and $\lambda_1$ in (\ref{Eq:EN-SPCA}) are related to $\sigma^2,\,\beta^2$ and $\gamma$. 
As aforementioned, our approach can be interpreted as defining an improper prior over $\bm{x}$, which promotes sparsity~\citep{Tipping-01}. We use
$p(\bm{x})\propto\prod_i\frac{1}{|x_i|+\varepsilon}$ (instead of
$\prod_i\exp(-\gamma|x_i|)$) as the prior such that
$\bm{x}|\beta^2,\varepsilon\sim\mathcal{G}(\bm{0},\beta^2\bm{I})p(\bm{x})$.
Replacing the prior in (\ref{Eq:EN-SPCA}) with our prior, results in
\begin{equation}\label{Eq:spcacompare}
\min_{\bm{x}}||\bm{y}_i-\bm{Qx}||^2_2+\lambda||\bm{x}||^2_2+\lambda_1\sum_i\log(|x_i|+\varepsilon).
\end{equation}
Since the problem in (\ref{Eq:spcacompare}) is equivalent to
(SGEV-A) with $\bm{B}=\bm{I}_n$, it is clear that DC-PCA can be expected to provide sparser solutions than SPCA because of the prior $p(\bm{x})$ that promotes sparsity.
It is to be noted that the SPCA framework is not
extendible to other settings like FDA or CCA unlike our formulation which is generic. 

\subsection{Comparison to GPower$_{\ell_0}$}
Consider the following regularized sparse PCA program where $\bm{A}\succeq 0$:
\begin{equation}\label{Eq:app}
\max\{\bm{x}^T\bm{Ax}-\tilde{\rho}\Vert\bm{x}\Vert_0:\bm{x}^T\bm{x}\le 1\}.
\end{equation}
In our case, we approximated $\Vert\bm{x}\Vert_0$ by $\Vert\bm{x}\Vert_\varepsilon$ and posed the approximate program as a d.c. program, resulting in DC-PCA obtained through majorization-minimization.
On the other hand, without using any approximations to $\Vert\bm{x}\Vert_0$, \citet{Journee-08} showed that the solution to (\ref{Eq:app}) can be obtained as:
\begin{equation}
x_i=\frac{\left[\text{sign}((\bm{c}^T_i\bm{z})^2-\tilde{\rho})\right]_+\bm{c}^T_i\bm{z}}{\sqrt{\sum^n_{i=1}\left[\text{sign}((\bm{c}^T_i\bm{z})^2-\tilde{\rho})\right]_+(\bm{c}^T_i\bm{z})^2}},\,\forall\,i,
\end{equation}
where $\bm{A}=\bm{C}^T\bm{C}$, $\bm{C}$ is a $p\times n$ matrix, $\bm{c}_i$ is the $i^{th}$ column of $\bm{C}$ and $\bm{z}$ is the solution to the following program which is the maximization of a convex function over an ellipsoid:
\begin{equation}\label{Eq:journee}
\max\left\{\sum^n_{i=1}\left[(\bm{c}^T_i\bm{z})^2-\tilde{\rho}\right]_+\,:\,\Vert\bm{z}\Vert^2_2=1,\,\bm{z}\in\bb{R}^p\right\}.
\end{equation}
(\ref{Eq:journee}) is then solved iteratively (called GPower$_{\ell_0}$) using a simple gradient-type scheme resulting in the following update rule:
\begin{eqnarray}
\bm{z}^{(l+1)}&=&\sum^n_{i=1}\left[\text{sign}((\bm{c}^T_i\bm{z}^{(l)})^2-\tilde{\rho})\right]_+\bm{c}^T_i\bm{z}^{(l)}\bm{c}_i,\nonumber\\
\bm{z}^{(l+1)}&=&\frac{\bm{z}^{(l+1)}}{\Vert\bm{z}^{(l+1)}\Vert_2}.\nonumber
\end{eqnarray}
It can be shown that their gradient-type scheme is equivalent to an MM-type algorithm.
Therefore, our method differs from GPower$_{\ell_0}$ only in a small way. This is also confirmed empirically in the following subsection where DC-PCA and GPower$_{\ell_0}$ exhibit similar performance. We would like to mention that unlike our sparse generalized eigenvalue algorithm (Algorithm 1), GPower$_{\ell_0}$ cannot be readily extended to settings like FDA and CCA.
\subsection{Experimental results}\label{Sec:Experiments}
In this section, we illustrate the effectiveness of DC-PCA in terms of sparsity and scalability on various datasets. On small datasets, the performance of DC-PCA is compared against SPCA, 
DSPCA, 
GSPCA and GPower$_{\ell_0}$, while on large datasets, DC-PCA is compared to all these algorithms except DSPCA and GSPCA due to scalability issues. Since GPower$_{\ell_0}$ has been compared to the greedy algorithm of \citet{Aspremont-08} by \citet{Journee-08}, wherein it is shown that these two algorithms perform similarly except for the greedy algorithm being computationally more complex, we do not include the greedy algorithm in our comparison. 
The results show that the performance of DC-PCA is comparable to the performance of many of these algorithms, but with \emph{better scalability}. The experiments in this paper are carried out on a Linux $3$ GHz, $4$ GB RAM workstation. 
On the implementation side, we fix $\varepsilon$ to be the machine precision in all our experiments, which is motivated from the discussion in Section~\ref{Sec:DC}.
\begin{table*}
\caption{Loadings for first three sparse principal components (PCs) of
the pit props data. The SPCA and DSPCA loadings are taken
from~\citet{Zou-06} and \citet{Aspremont-07} respectively.}\vspace{-5mm}
\begin{center}\footnotesize{
\begin{tabular}[t]{ccccccccccccccc}\hline
& PC& $x_1$ & $x_2$ & $x_3$ &$x_4$ & $x_5$ & $x_6$ & $x_7$
& $x_8$ & $x_9$ & $x_{10}$ & $x_{11}$ & $x_{12}$ & $x_{13}$\\
\hline & 1& -.48 & -.48& 0 & 0 & .18 & 0 & -.25 & -.34 &
-.42 &
-.40 & 0 & 0 & 0\\
SPCA & 2 & 0 & 0 & .79 & .62 & 0 & 0 & 0 & -.02 & 0 & 0 & 0 &
.01 & 0\\
& 3 & 0 & 0 & 0 & 0 & .64 & .59 & .49 & 0 & 0 & 0 & 0 & 0 &
-.02\\ \hline & 1 & -.56 & -.58 & 0 & 0 & 0 & 0 & -.26 & -.10
& -.37 & -.36 & 0 & 0 & 0\\
DSPCA & 2 & 0 & 0 & .71 & .71 & 0 & 0 & 0 & 0 & 0 & 0 & 0 & 0 &
0 \\
& 3 & 0 & 0 & 0 & 0 & 0 & -.79 & -.61 & 0 & 0 & 0 & 0 & 0 &.01\\
\hline 
& 1 & .44 &  .45 & 0 & 0 &0 & 0 & .38 & .34 &
.40 &  .42 & 0 & 0 &0\\
GSPCA & 2 &0 & 0& .71 & .71& 0 & 0 & 0 & 0 & 0 & 0 & 0 & 0 &
0 \\
&3 & 0 & 0 & 0 & 0 & 0 & .82 & .58 & 0 & 0 & 0 & 0 & 0 & 0
\\ \hline
& 1 & .44 &  .45 & 0 & 0 &0 & 0 & .38 & .34 &
.40 &  .42 & 0 & 0 &0\\
GPower$_{\ell_0}$ & 2 &0 & 0& .71 & .71& 0 & 0 & 0 & 0 & 0 & 0 & 0 & 0 &
0 \\
&3 & 0 & 0 & 0 & 0 & 0 & .82 & .58 & 0 & 0 & 0 & 0 & 0 & 0
\\ \hline & 1 & .45 &  .46 & 0 & 0 &0 & 0 & .37 & .33 &
.40 &  .42 & 0 & 0 &0\\
DC-PCA & 2 &0 & 0& .71 & .71& 0 & 0 & 0 & 0 & 0 & 0 & 0 & 0 &
0 \\
&3 & 0 & 0 & 0 & 0 & 0 & .82 & .58 & 0 & 0 & 0 & 0 & 0 & 0
\\ \hline
\end{tabular}}
\label{tab:loadings}
\end{center}
\vspace{-4mm}
\end{table*}
\subsubsection{Pit props data}
The pit props dataset \citep{Jeffers-67} has become a standard
benchmark example to test sparse PCA algorithms. The first 6
principal components (PCs) capture $87\%$ of the total variance. Therefore, the explanatory power of sparse PCA methods is often compared on the first 6 sparse PCs.\footnote{The discussion so far dealt with computing the first sparse eigenvector of $\bm{A}$. 
To compute the subsequent sparse eigenvectors, usually the sparse PCA algorithm is applied to a sequence of deflated matrices. See \citet{Mackey-08} for details. In this paper, we used the orthogonalized Hotelling's deflation as mentioned in \citet{Mackey-08}.}
Table~\ref{tab:loadings} shows the first 3 sparse PCs
and their loadings for SPCA, DSPCA, GSPCA, GPower$_{\ell_0}$ and DC-PCA. Using the first 6 sparse PCs, SPCA captures 75.8\% of the variance with a cardinality pattern of
$(7,4,4,1,1,1)$, which indicates the number of non-zero loadings for the first to the sixth sparse PC, respectively. This results in a total of 18 non-zero loadings for SPCA, while DSPCA
captures 75.5\% of the variance with a sparsity pattern of $(6,2,3,1,1,1)$,
totaling 14 non-zero loadings. With a sparsity pattern of
$(6,2,2,1,1,1)$ (total of \emph{only} $13$ non-zero loadings), DC-PCA, GSPCA and GPower$_{\ell_0}$ can capture 77.1\% of the total variance. 
Comparing the cumulative variance and cumulative cardinality, Figures~\ref{fig:pitprop}(a--b) show that DC-PCA explains more variance with fewer non-zero loadings than SPCA and DSPCA. In addition, its performance is similar to that of GSPCA and GPower$_{\ell_0}$. For the first sparse PC,
Figure~\ref{fig:pitprop}(c) shows that DC-PCA consistently
explains more variance with better sparsity than SPCA, while performing similar to other algorithms. Figure~\ref{fig:pitprop}(d) shows the variation of sparsity and explained variance with respect to $\tilde{\rho}$
for the first sparse PC computed with DC-PCA. This plot summarizes the method for setting $\tilde{\rho}$: the algorithm is run for various $\tilde{\rho}$ and the value of $\tilde{\rho}$ that achieves the desired sparsity is selected.
\begin{figure*}[t]
  \centering
  \begin{tabular}{ccc}
    \begin{minipage}{8cm}
      \center{\epsfxsize=7cm
     \epsffile{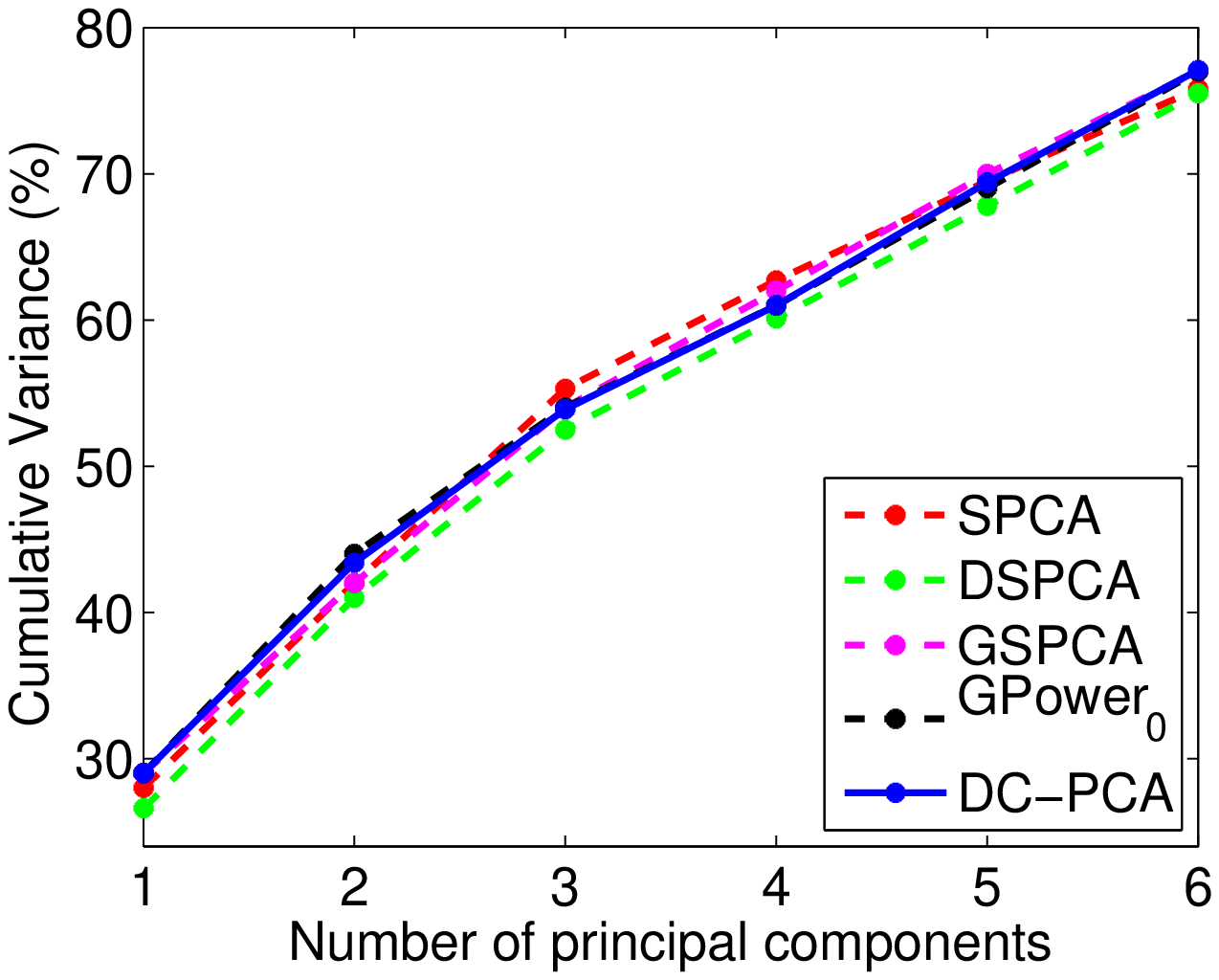}}\vspace{-4mm}
      {\small \center{(a)}}
    \end{minipage}
    \begin{minipage}{8cm}
      \center{\epsfxsize=7cm
	   \epsffile{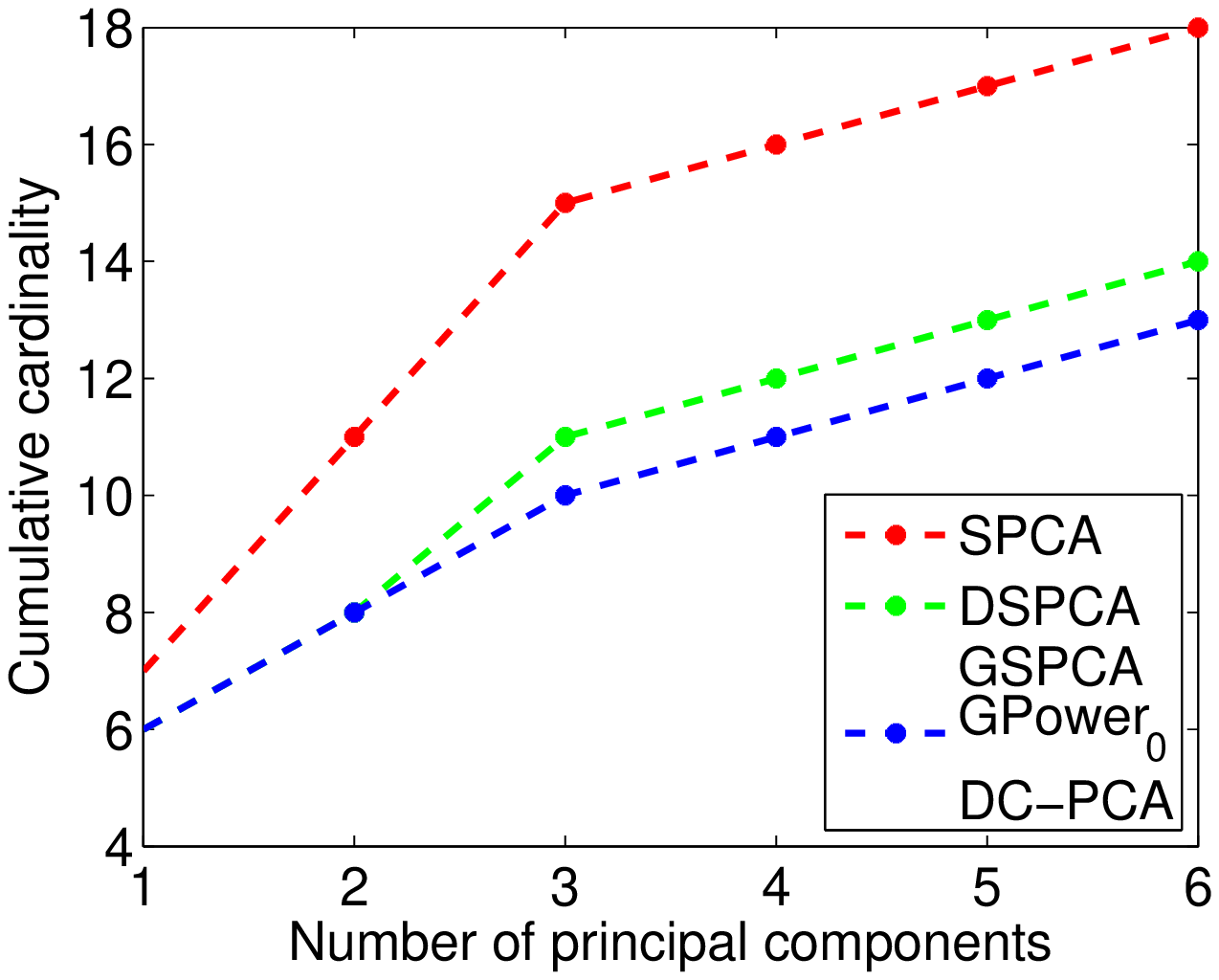}}\vspace{-4mm}
      {\small \center{(b)}}
    \end{minipage}
    \vspace{5mm}
  \end{tabular}
  \begin{tabular}{ccc}
    \begin{minipage}{8cm}
      \center{\epsfxsize=7cm
      \epsffile{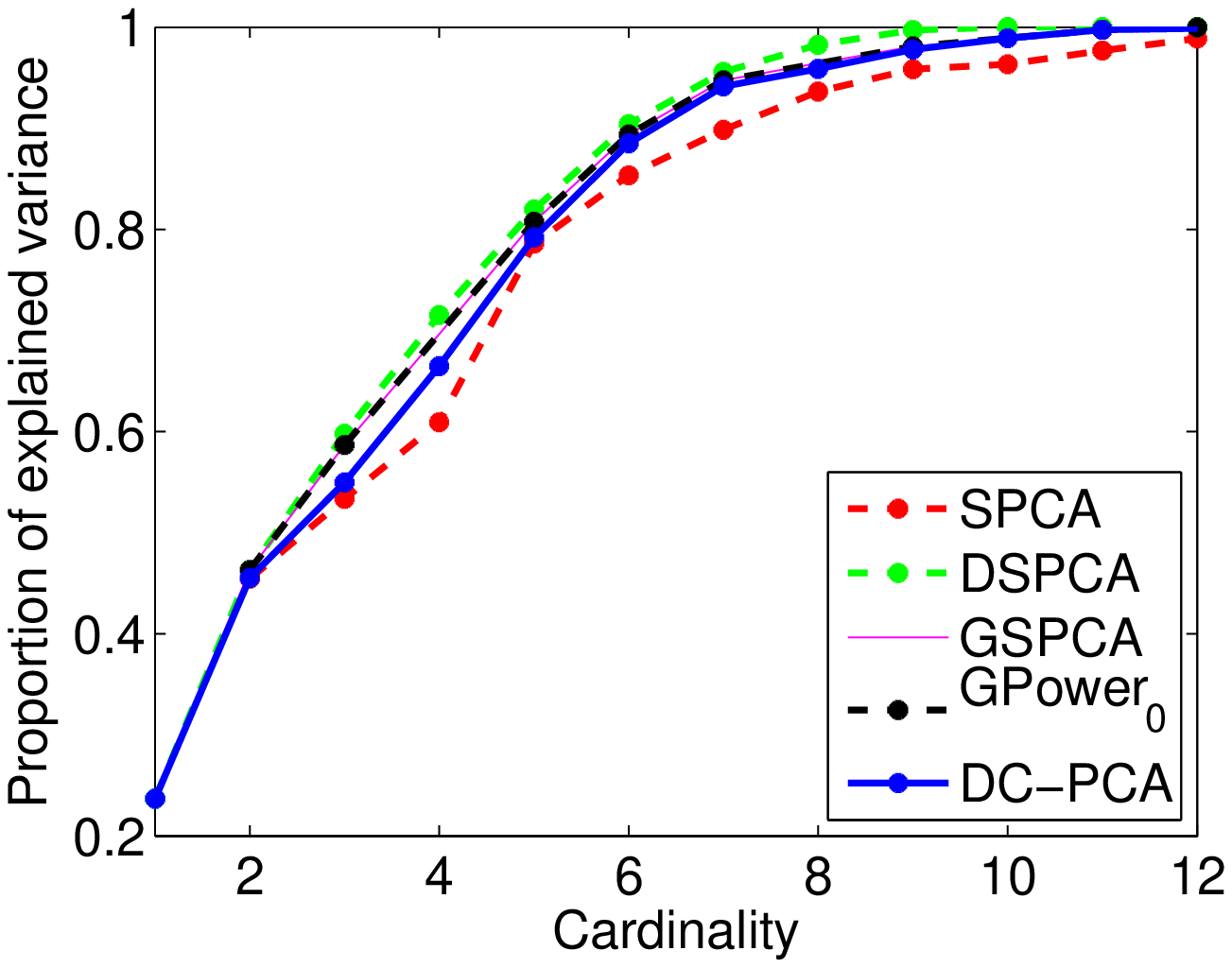}}\vspace{-4mm}
      {\small \center{(c)}}
    \end{minipage}
    \begin{minipage}{8cm}
      \center{\epsfxsize=7cm
\epsffile{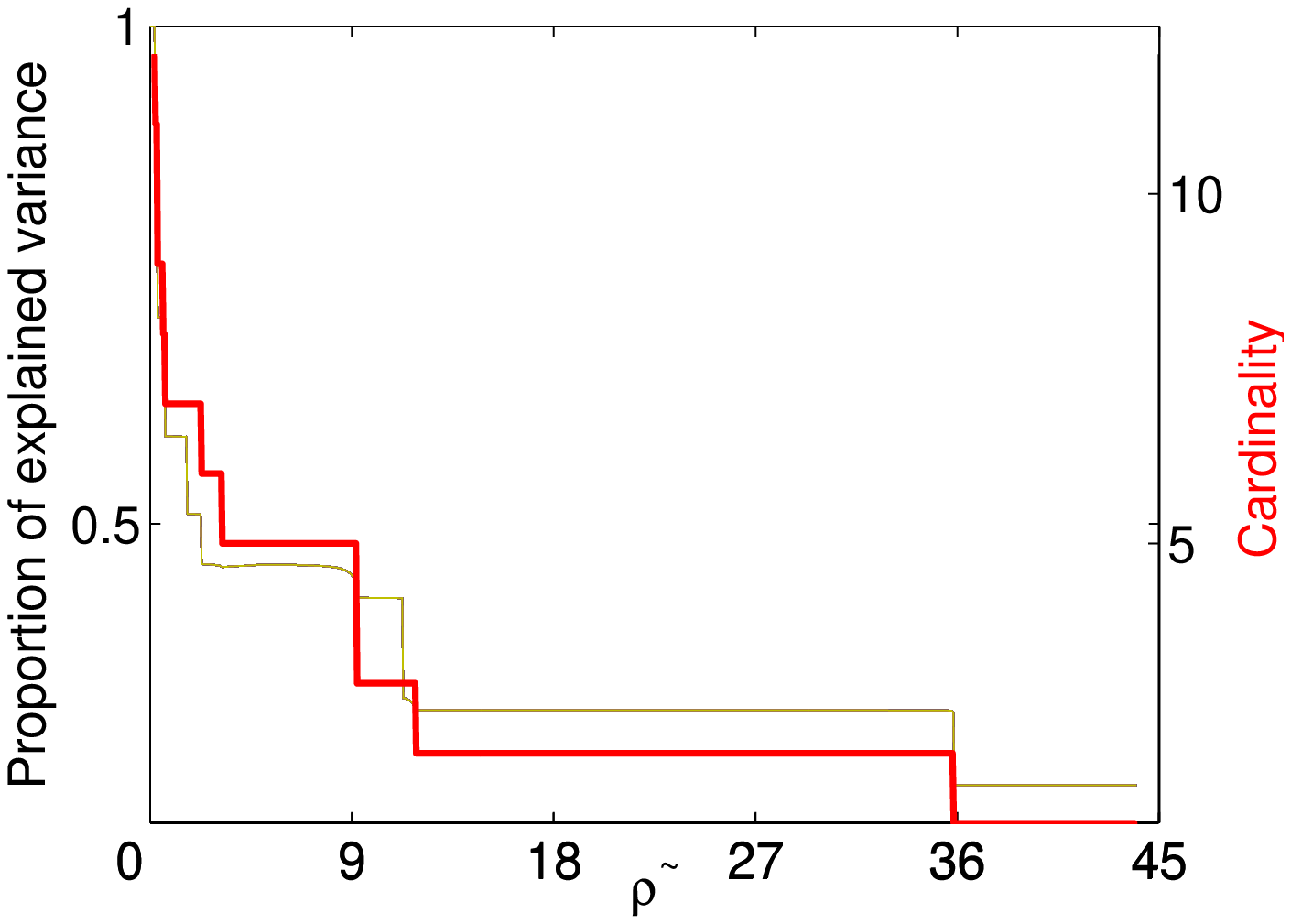}}\vspace{-4mm}
      {\small \center{(d)}}
    \end{minipage}
    \vspace{0mm}
  \end{tabular}
  \caption{Pit props: (a) cumulative variance and (b) cumulative cardinality for the first 6 sparse principal components (PCs); (c) proportion of explained variance (PEV) vs. cardinality for the first sparse PC (obtained by varying the sparsity parameter and computing the cardinality and explained variance for the solution vector); 
(d) dependence of sparsity and PEV on $\tilde{\rho}$
for the first sparse PC computed with DC-PCA. 
\vspace{-2mm}}
  \label{fig:pitprop}
  \vspace{-5mm}
\end{figure*}

\subsubsection{Random test problems}
In this section, we follow the experimental setup that is considered in \citet{Journee-08}. Throughout this section, we assume $\bm{A}=\bm{C}^T\bm{C}$, where $\bm{C}$ is a $p\times n$ random matrix whose entries are generated according to a Gaussian distribution, with zero mean and unit variance.
In the following, we present the trade-off curves (proportion of explained variance vs. cardinality for the first sparse PC associated with $\bm{A}$), computational complexity vs. cardinality and computational complexity vs. problem size for various sparse PCA algorithms.\\\\
\textbf{Trade-off curves.} Figure~\ref{fig:random}(a) shows the trade-off between the proportion of explained variance and the cardinality for the first sparse PC associated with $\bm{A}$ for various sparse PCA algorithms. For each algorithm, the sparsity inducing parameter ($k$ in the case of DSPCA and GSPCA, and the regularization parameter in the case of SPCA, GPower$_{\ell_0}$ and DC-PCA) is incrementally increased to obtain the solution vector with cardinality that decreases from $n$ to 1. The results displayed in Figure~\ref{fig:random}(a) are averages of computations on 100 random matrices with dimensions $p=100$ and $n=300$. It can be seen from Figure~\ref{fig:random}(a) that DC-PCA performs similar to DSPCA, GSPCA and GPower$_{\ell_0}$, while performing better than SPCA.\\\\
\begin{figure*}[t]
  \centering
  \begin{tabular}{ccc}
    \begin{minipage}{8cm}
      \center{\epsfxsize=7cm
     \epsffile{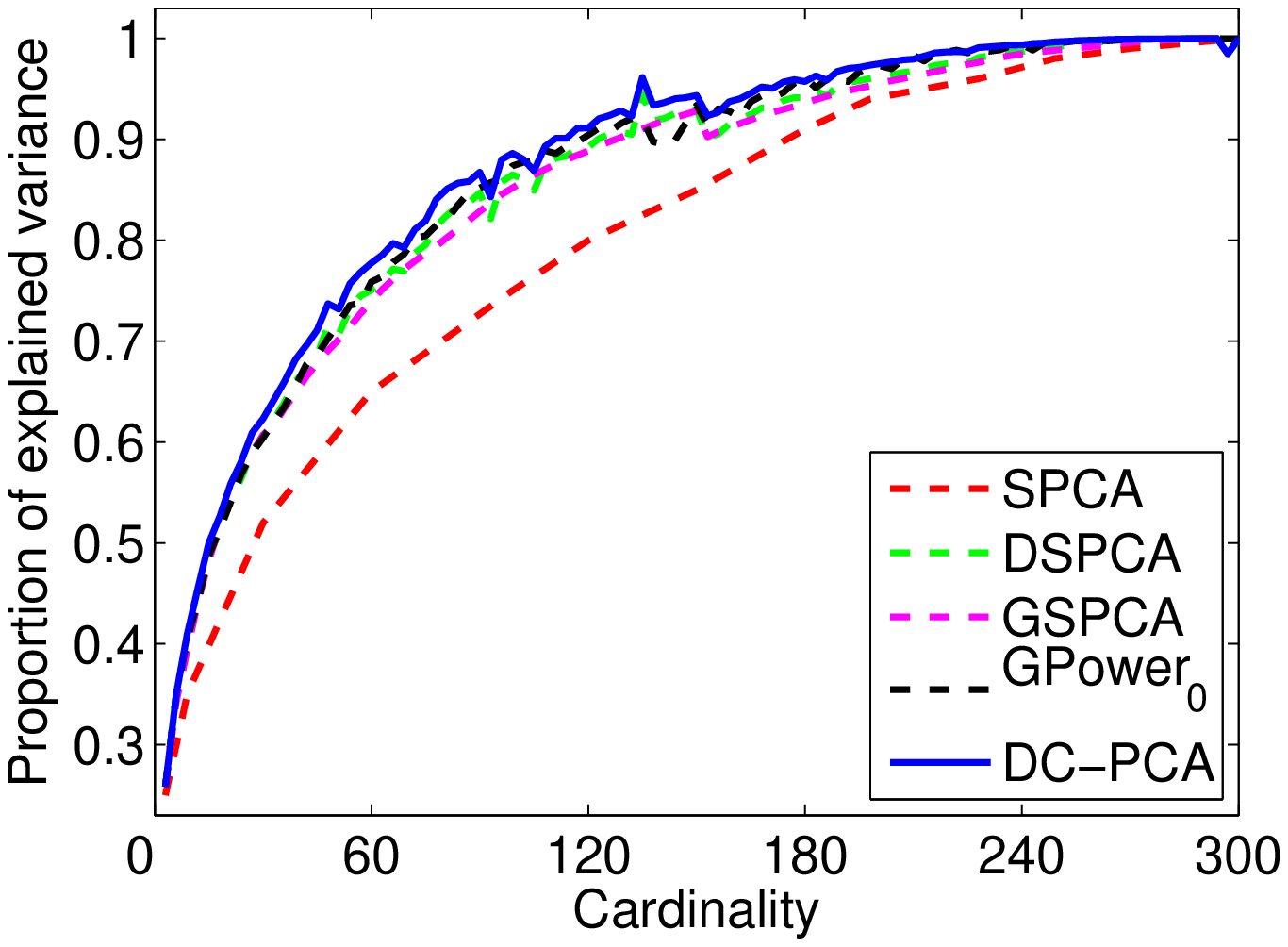}}\vspace{-4mm}
      {\small \center{(a)}}
    \end{minipage}
    \begin{minipage}{8cm}
      \center{\epsfxsize=7cm
	   \epsffile{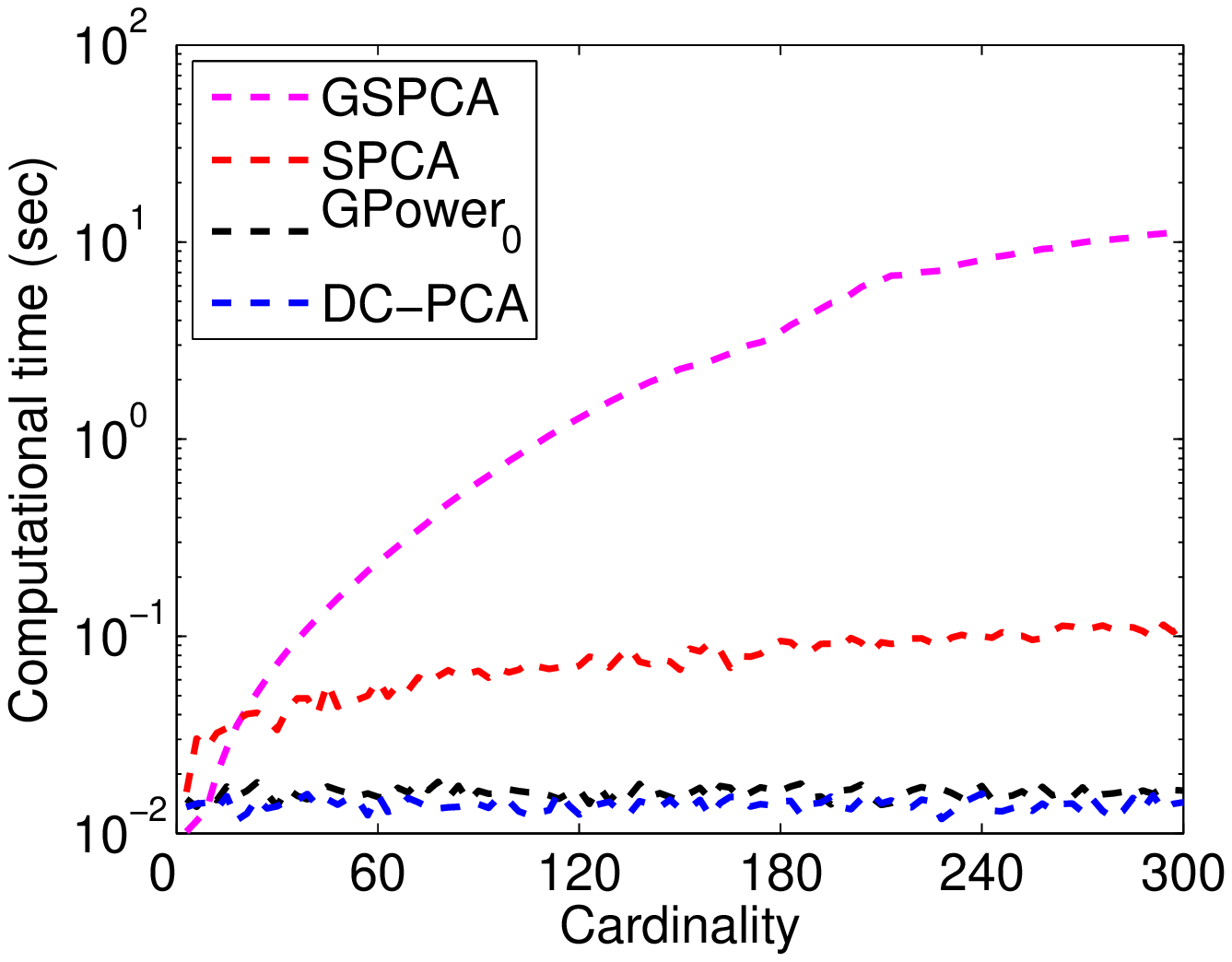}}\vspace{-4mm}
      {\small \center{(b)}}
    \end{minipage}
    \vspace{0mm}
  \end{tabular}
  \caption{Random test data: (a) (average) proportion of explained variance vs. cardinality for the first sparse PC of $\bm{A}$; (b) (average) computation time vs. cardinality. In (a), all the sparse PCA algorithms perform similarly and better than SPCA. In (b), the complexity of GSPCA grows significantly with increasing cardinality of the solution vector, while the speed of the other methods is almost independent of the cardinality.
}  \label{fig:random}
\vspace{-7mm}
\end{figure*}
\textbf{Computational complexity vs. Cardinality.} Figure~\ref{fig:random}(b) shows the average time required by the sparse PCA algorithms to extract the first sparse PC of $\bm{A}$ with $p=100$ and $n=300$, for varying cardinality. It is obvious from Figure~\ref{fig:random}(b) that as the cardinality increases, GSPCA tends to get slower while the speed of SPCA, GPower$_{\ell_0}$ and DC-PCA is not really affected by
the cardinality. We did not show the results of DSPCA in Figure~\ref{fig:random}(b) as its computational complexity is an order of magnitude (around 100 times) more than the scale on the vertical axis of Figure~\ref{fig:random}(b). \citet{Journee-08} have demonstrated that the greedy method proposed by \citet{Aspremont-08} also exhibits the behavior of increasing computational complexity with the increase in cardinality.
\\\\
\textbf{Computational complexity vs. Problem size.} Figure~\ref{fig:timevsproblem} shows the average computation time in seconds, required by various sparse PCA algorithms, to extract the first sparse PC of $\bm{A}$, for various problem sizes, $n$, where $n$ is increased exponentially and $p$ is fixed to 500. The times shown are averages over 100 random instances of $\bm{A}$
for each problem size, where the sparsity inducing parameters are chosen such that the solution vectors of these algorithms exhibit comparable cardinality. It is clear from Figure~\ref{fig:timevsproblem} that DC-PCA and GPower$_{\ell_0}$ scale better to large-dimensional problems than the other algorithms. Since, on average, GSPCA and DSPCA 
are much slower than the other methods, even for low cardinalities (see Figure~\ref{fig:random}(b)), we discard them from all the following numerical experiments that deal with large $n$.
\par For the remaining algorithms, SPCA, GPower$_{\ell_0}$ and DC-PCA, we run another round of experiments, now examining the computational complexity with varying $n$ and $p$ but with a fixed aspect ratio $n/p=10$. The results are depicted in Table~\ref{tab:time-1}.
Again, the corresponding regularization parameters are set in such a way that the solution vectors of these algorithms exhibit comparable cardinality. The values displayed in Table~\ref{tab:time-1} correspond to the average running times of the algorithms on 100 random instances of $\bm{A}$ 
for each problem size. It can be seen that our proposed method, DC-PCA, is comparable to GPower$_{\ell_0}$ and faster than SPCA.
\begin{figure*}[t]
\centering
  \begin{tabular}{ccc}
    \begin{minipage}{8cm}
      \center{\epsfxsize=7cm
      \epsffile{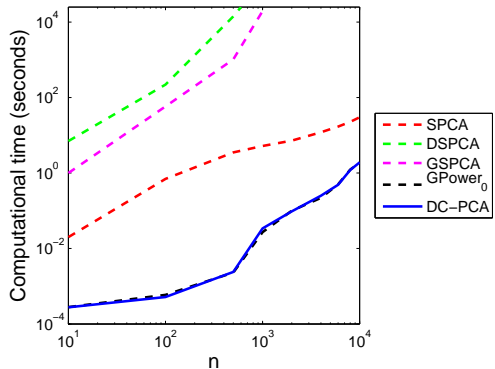}}\vspace{-4mm}
    \end{minipage}
    \vspace{3mm}
  \end{tabular}
  \caption{Average computation time (seconds) for the first sparse PC of $\bm{A}$ vs. problem size, $n$, over 100 randomly generated matrices $\bm{A}$.}\vspace{-4mm}
  \label{fig:timevsproblem}
\end{figure*}
\begin{table*}
\caption{Average computation time (in seconds) for the first sparse PC associated with $\bm{A}$ for a fixed regularization parameter.}
\begin{center}
\begin{tabular}[t]{cccccc}\hline
$p\times n$ & $100\times 1000$ & $250\times 2500$ & $500\times 5000$& $750\times 7500$ & $1000\times 10000$\\ \hline
SPCA & 0.135 & 1.895 & 10.256 & 34.367 & 87.459\\
GPower$_{\ell_0}$ & 0.027 & 0.159 & 0.310 & 1.224 & 1.904\\
DC-PCA & 0.034 & 0.151 & 0.301 & 1.202 & 1.913\\
\hline
\end{tabular}
\label{tab:time-1}
\end{center}
\vspace{-4mm}
\end{table*}

\subsubsection{Gene expression data}
Gene expression data from DNA microarrays provides the expression level of thousands of genes across several hundreds or thousands of experiments. To enhance the interpretation of these large data sets, sparse PCA algorithms can be applied, to extract sparse principal components that involve only a few genes. \\\\
\textbf{Datasets.} Usually, gene expression data is specified by a $p\times n$ matrix (say $\bm{C}$) of $p$ samples and $n$ genes. The covariance matrix, $\bm{A}$ is therefore computed as $\bm{C}^T\bm{C}$. In our experiments, we consider three gene expression datasets which are tabulated in Table~\ref{tab:dataset}.
\begin{table*}
\caption{Gene expression datasets}
\begin{center}
\begin{tabular}[t]{cccc}\hline
Dataset & Samples ($p$) & Genes ($n$) & Reference\\ \hline
Colon cancer & 62 & 2000 & \citet{Alon-99}\\
Leukemia & 38 & 7129 & \citet{Golub-99}\\
Ramaswamy & 127 & 16063 & \citet{Ramaswamy-01}\\
\hline
\end{tabular}
\label{tab:dataset}
\end{center}
\vspace{-4mm}
\end{table*}
The colon cancer dataset~\citep{Alon-99} consists of 62 tissue samples (22 normal and 40 cancerous) with the gene expression profiles of $n=2000$ genes extracted from DNA microarray data. Its first principal component explains 44.96\% of the total variance. The leukemia dataset~\citep{Golub-99} consists of a training set of 38 samples (27 ALL and 11 AML, two variants of leukemia) from bone marrow specimens and a test set of 34 samples (20 ALL and 14 AML). This dataset has been used widely in classification settings where the goal is to distinguish between two variants of leukemia. All samples have 7129 features, each of which corresponds to a normalized gene expression value extracted from the microarray image. The first principal component explains 87.64\% of the total variance. The Ramaswamy dataset has 16063 genes and 127 samples, its first principal component explaining 76.5\% of the total variance. 
\par The high dimensionality of these datasets makes them suitable candidates
for studying the performance of sparse PCA algorithms, by investigating their ability to explain variance in the data based on a small number of genes, to obtain interpretable results.
Since DSPCA and GSPCA are not scalable for these large datasets, in our study, we compare DC-PCA to SPCA and GPower$_{\ell_0}$. \\\\
\textbf{Trade-off curves.} Figures~\ref{fig:cancer-leukemia-ramaswamy}(a-c) show the proportion of explained variance versus the cardinality for the first sparse PC for the datasets shown in Table~\ref{tab:dataset}. It can be seen that DC-PCA performs similar to GPower$_{\ell_0}$ and performs better than SPCA. 
\\\\
\textbf{Computational complexity.} The average computation time required by the sparse PCA algorithms on each dataset is shown in Table~\ref{tab:datasettimes}. The indicated times are averages over $n$ computations, one for each cardinality ranging from $n$ down to $1$.
The results show that DC-PCA and GPower$_{\ell_0}$ are significantly faster than SPCA, which, for a long time, was widely accepted as the algorithm that can handle large datasets.\\\\
Overall, the results in this section demonstrate that DC-PCA performs similar to GPower$_{\ell_0}$,
the state-of-the-art, and better than SPCA, both in terms of scalability and proportion of variance explained vs. cardinality. We would like to mention that our sparse PCA algorithm (DC-PCA) is derived from a more general framework, that can be used to address other generalized eigenvalue problems as well, e.g., sparse CCA, sparse FDA, etc.
\begin{figure*}[t]
  \centering
  \begin{tabular}{ccc}
    \begin{minipage}{8cm}
      \center{\epsfxsize=7cm
      \epsffile{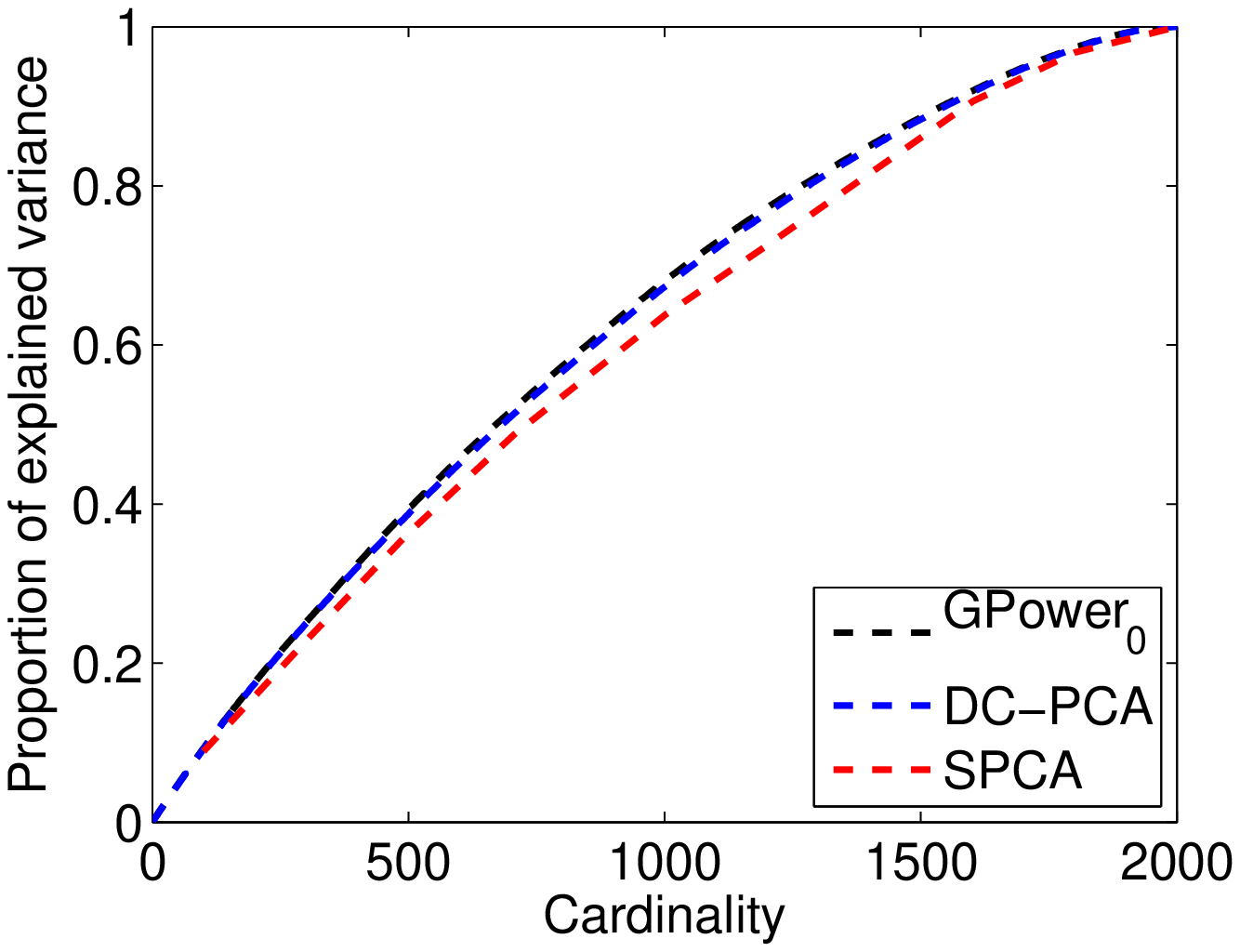}}\vspace{-4mm}
      {\small \center{(a)}}
    \end{minipage}
    \begin{minipage}{8cm}
      \center{\epsfxsize=7cm
      \epsffile{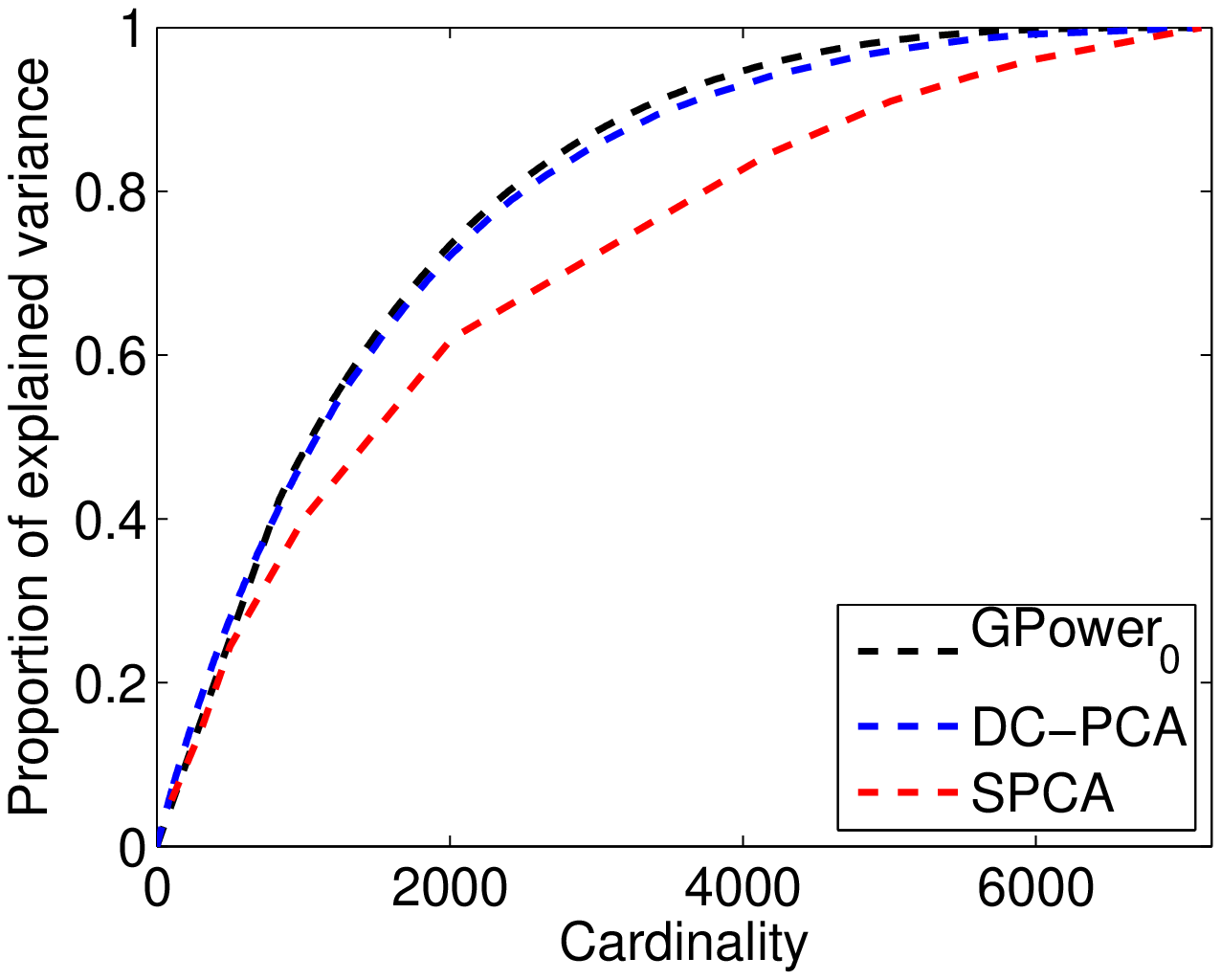}}\vspace{-4mm}
      {\small \center{(b)}}
    \end{minipage}
    \vspace{5mm}
  \end{tabular}
  \begin{tabular}{c}
    \begin{minipage}{8cm}
      \center{\epsfxsize=7cm
       \epsffile{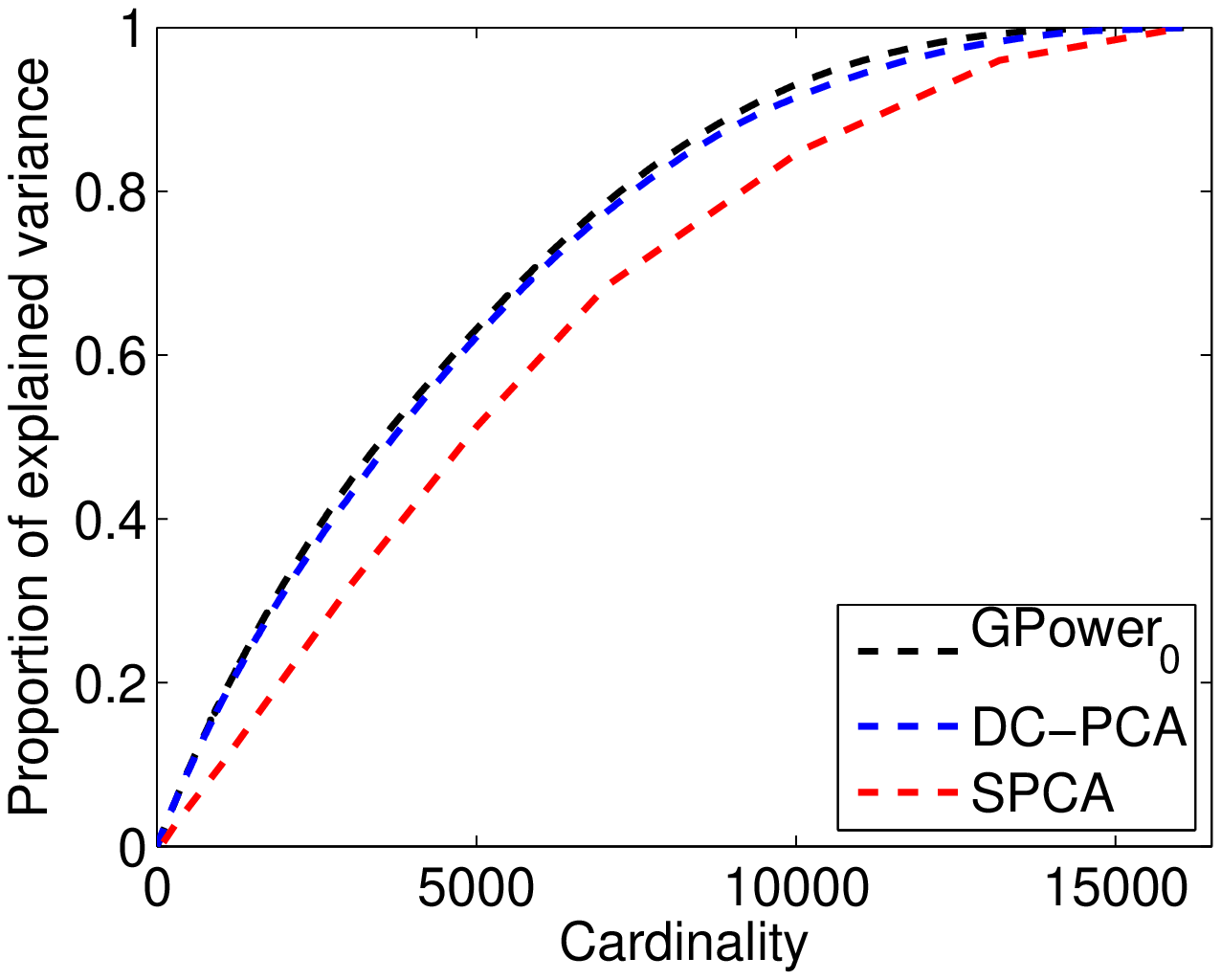}}\vspace{-4mm}
      {\small \center{(c)}}
    \end{minipage}
    \vspace{0mm}
  \end{tabular}
  \caption{Trade-off curves between explained variance and cardinality for (a) Colon cancer, (b) Leukemia and (c) Ramaswamy datasets. The proportion of variance explained is computed on the first sparse principal component. (a--c) show that DC-PCA performs similar to GPower$_{\ell_0}$, while explaining more variance (for a fixed cardinality) than SPCA.}\vspace{-6mm}
  \label{fig:cancer-leukemia-ramaswamy}
\end{figure*}
\begin{table*}
\caption{Computation time (in seconds) to obtain the first sparse PC, averaged over cardinalities ranging from 1 to $n$, for the Colon cancer, Leukemia and Ramaswamy datasets.}
\begin{center}
\begin{tabular}[t]{cccc}\hline
 & Colon cancer & Leukemia & Ramaswamy\\ 
$n$ & 2000& 7129& 16063\\\hline
SPCA & 2.057 & 3.548 & 38.731\\
GPower$_{\ell_0}$ & 0.182 & 0.223 & 2.337\\
DC-PCA & 0.034 & 0.156 & 0.547\\
\hline
\end{tabular}
\label{tab:datasettimes}
\end{center}
\vspace{-3mm}
\end{table*}

\section{Sparse Canonical Correlation Analysis}\label{sec:cca}
In this section, we consider sparse CCA as a special case of the sparse GEV algorithm and present two CCA applications where sparsity is helpful. We call our sparse CCA algorithm DC-CCA, where $\bm{A}$ and $\bm{B}$ are determined from the covariance and cross-covariance matrices as explained right below (\ref{Eq:CCA}). Note that $\bm{A}$ is indefinite and, therefore, in our experiments, we choose $\tau=-\lambda_{min}(\bm{A})$ in Algorithm~\ref{alg1}. In the following, we present two sparse CCA applications, one related to the task of cross-language document retrieval and the other dealing with semantic annotation and retrieval of music~\citep{Torres-07a, Torres-07b}. 
\par The sparse CCA algorithm considered in this section was earlier proposed by us in \citet{Torres-07b}. Related work involves the sparse CCA algorithm due to \citet{Hardoon-08}. 
In Section~\ref{Sec:Sparse-GEV}, we presented a SDP relaxation, which could be applied for  sparse CCA. However, in this section, we use DC-CCA (based on Algorithm~\ref{alg1}) to perform sparse CCA, as it scales better for large problem sizes. We illustrate its performance in the above mentioned applications.
\subsection{Cross-language document retrieval}
The problem of cross-language document retrieval involves a collection of documents, $\{D_i\}^N_{i=1}$ with each document being represented in different languages, say English and French. The goal of the task is, given a query string in one language, retrieve the most relevant document(s) in the target language. The first step is to obtain a semantic representation of the documents in both languages, which models the correlation between translated versions, so we can detect similarities in content between the two document spaces (one for English and the other for French). This is exactly what CCA does by finding a low-dimensional representation in both languages, with maximal correlation between them. \citet{Vinokourov-03} used CCA to address this problem and showed that the CCA approach performs better than the latent semantic indexing approach used by \citet{Littman-98}. CCA provides an \emph{efficient} basis representation (that captures the maximal correlation) for the two document spaces.
\par Using a bag-of-words representation for the documents, sparse CCA would allow to find a low-dimensional model based on a small subset of words in both languages. This would improve the interpretability of the model and could identify small subsets of words that are used in similar contexts in both languages and, possibly, are translations of one another. Representing documents by their similarity to all other documents (e.g., by taking inner products of bag-of-word vectors, as explained below), sparse CCA would create a low-dimensional model that only requires to measure the similarity for a small subset of the training documents. This would immediately improve storage requirements and the efficiency of retrieval computations.
In this study, we follow the second approach, representing documents by their similarity to all other training documents by applying a linear kernel function to a binary bag-of-words representation of the documents, as proposed in~\citet{Vinokourov-03}. This will illustrate how we can achieve significant sparsity without significant loss of retrieval performance. 
\par More specifically, each version of a document (English or French) is modeled using a bag-of-words feature vector. 
Within a feature vector, we associate an element in $\{0,1\}$ with each word $w_i$ in its language vocabulary. A value of 1 indicates that $w_i$ is found in the document. We collect the feature vectors into the $N\times P$ matrix $\bm{E}$, where we collect the English feature vectors, and the $N\times Q$ matrix $\bm{F}$, where we collect the French feature vectors. $N$ is the number of documents and $P$ and $Q$ are the vocabulary sizes of $\bm{E}$ and $\bm{F}$ respectively. Computing the similarity between English documents as the inner product between their binary bag-of-words vectors (i.e., the rows of $\bm{E}$) results in computing an $N\times N$ data matrix $\bm{E}\bm{E}^T$. Similarly, we compute an $N\times N$ data matrix $\bm{F}\bm{F}^T$ and obtain two feature spaces which are both $N$-dimensional. \par By applying sparse CCA, we effectively perform simultaneous feature selection across two vector spaces and characterize the content of and correlation between English and French documents in an efficient manner. We use the DC-CCA algorithm, using the covariance and cross-variance matrices associated with the document matrices $\bm{E}\bm{E}^T$ and $\bm{F}\bm{F}^T$ and obtain successive pairs of sparse canonical components which we stack into the columns of $\bm{V}_E$ and $\bm{V}_F$. (Subsequent pairs of these sparse canonical components are obtained by deflating $\bm{E}\bm{E}^T$ and $\bm{F}\bm{F}^T$ with respect to previous canonical components. For a detailed review on deflation, we refer the reader to \citet{ShaweTaylor-04}.)

Then, given a query document in an input language, say English, we convert the query into the appropriate feature vector, $\bm{q}_E$. We project $\bm{q}_E$ onto the subspace spanned by the sparse canonical components in the English language space by computing $\bm{V}^T_E\bm{q}_E$\footnote{Notice how this projection, onto the sparse canonical components, only requires to compute a few elements of  $\bm{q}_E$, i.e., the ones corresponding to the non-zero loadings of the sparse canonical components; differently said, we only need to compute the similarity of the query document to a small subset of all training documents.}. Similarly, we project all the French 
documents onto the subspace spanned by the sparse canonical components, $\bm{V}_F$ associated with the French language. Finally, we perform document retrieval by selecting those French documents whose projections are closest to the projected query, where we measure distance in a nearest neighbor sense.

\subsubsection{Experimental Details} 
The data set used was the Aligned Hansards of the 36th Parliament of Canada \citep{Germann-01}, which is a collection of 1.3 million pairs of text chunks (sentences or smaller fragments) aligned into English and French translations. The text chunks are split into documents based on $\ast\ast\ast$ delimiters. After removing stop words and rare words (those that occur less than 3 times), we are left with an $1800\times 26328$ English document-by-term matrix and a $1800 \times 30167$ French matrix. Computing $\bm{E}\bm{E}^T$ and $\bm{F}\bm{F}^T$ results in matrices of size $1800\times 1800$.
\par To generate a query, we select English test documents from a test set not used for training. The appropriate retrieval result is the corresponding French language version of the query document. To perform retrieval, the query and the French test documents are projected onto the sparse canonical components and retrieval is performed as described before.
\begin{table*}
\caption{Average area under the ROC curve (in \%) using CCA and sparse CCA (DC-CCA) in a cross-language document retrieval task. $d$ represents the number of canonical components and \emph{sparsity} represents the percentage of 
zero loadings in the canonical components.}\vspace{-5mm}
\begin{center}
\begin{tabular}[t]{cccccc}\hline
$d$ & 100 & 200 & 300 & 400 & 500\\ \hline 
CCA & 99.92 & 99.93 & 99.96 & 99.95 & 99.93\\
DC-CCA & 95.72 & 97.57 & 98.45 & 98.75 & 99.04\\
Sparsity & 87.15 & 87.56 & 87.95 & 88.21 & 88.44\\
\hline
\end{tabular}
\label{tab:language-cca}
\end{center}
\vspace{-6mm}
\end{table*}
Table~\ref{tab:language-cca} shows the performance of DC-CCA (sparse CCA) against CCA. We measure our results using the average area under the ROC curve (average AROC). The results in Table~\ref{tab:language-cca} are shown in percentages. To go into detail, for each test query we generate an ROC curve from the ranked retrieval results. Results are ranked according to their projected feature vector's Euclidean distance from the query. The area under this ROC curve is used to measure performance. For example, if the first returned document was the most relevant (i.e., the corresponding French language version of the query document) this would result in an ROC with area under the curve (AROC) of 1. If the most relevant document came in above the $75^{th}$ percentile of all documents, this would lead to an AROC of 0.75, and so on. So, we're basically measuring how highly the corresponding French language document ranks in the retrieval results. For a collection of queries we take the simple average of each query's AROC to obtain the average AROC. An average AROC of 1 is best, a value of 0.5 is as good as chance.
\par In Table~\ref{tab:language-cca}, we compare retrieval using sparse CCA to regular CCA. For sparse CCA, we use a sparsity parameter that leads to loadings that are approximately 10\% of the full dimensionality, i.e., the canonical components are approximately 90\% sparse. We note that sparse CCA is able to achieve good retrieval rates, only slightly sacrificing performance compared to regular CCA. This is the key result of this section: we can achieve performance close to regular CCA, by using only about 12\% of the number of loadings (i.e., documents) required by regular CCA. This shows that sparse CCA can narrow in on the most informative dimensions exhibited by data and can be used as an effective dimensionality reduction technique.
\subsection{Vocabulary selection for music information retrieval}
In this subsection we provide a short summary of the results in~\citet{Torres-07a}, which nicely illustrate how sparse CCA can be used to improve the performance of a statistical musical query application, by identifying problematic query words and eliminating them from the model. The 
application involves a computer audition system \citep{Turnbull-08} that can annotate songs with semantically meaningful words or \emph{tags} (such as, e.g., {\it rock} or {\it mellow}), or retrieve songs from a database, based on a semantic query. This system is based on a joint probabilistic model between words and acoustic signals, learned from a training data set of songs and song tags. ``Noisy" words, that are not or only weakly related to the musical content, will decrease the system's performance and waste computational resources. Sparse CCA is employed to prune away those noisy words and improve the system's performance.
\par The details of this experiment are beyond the scope of this work and can be found in ~\citet{Torres-07a}. In short, each song from the CAL-500 dataset\footnote{The CAL-500 data set consists of a set of songs, annotated with semantic tags, obtained by conducting human surveys. More details can be found in \citet{Turnbull-08}.} is represented in two different spaces: in a semantic space, based on a bag-of-words representation of a song's semantic tags, and in an audio space, based on Mel-frequency cepstral coefficients~\citep{Mckinney-03} extracted from a song's audio content. This representation allows sparse CCA to identify a small subset of words spanning a semantic subspace that is highly correlated with audio content. In Figure \ref{fig:aroc2}, we use sparse CCA to generate a sequence of vocabularies of progressively smaller size, 
ranging from full size (containing about 180 words) to very sparse (containing about 20 words),  depicted on the horizontal axis. For each vocabulary size, the computer audition system is trained and the average area under the receiver operating characteristic curve (AROC) is shown on the vertical axis, measuring its retrieval performance on an independent test set. The AROC (ranging between 0.5 for randomly ranked retrieval results and 1.0 for a perfect ranking) initially clearly improves, as sparse CCA (DC-CCA) generates vocabularies of smaller size: it is effectively removing noisy words that are detrimental for the system's performance. Also shown in Figure \ref{fig:aroc2} are the results of training the music retrieval system based on two alternative vocabulary selection techniques: random selection (offering no improvement) and a heuristic that eliminates words exhibiting less agreement amongst the human subjects that were surveyed to collect the CAL-500 dataset (only offering a slight improvement, initially).
\begin{figure}[t]
  \centering
      \center{\epsfxsize=8cm
      \epsffile{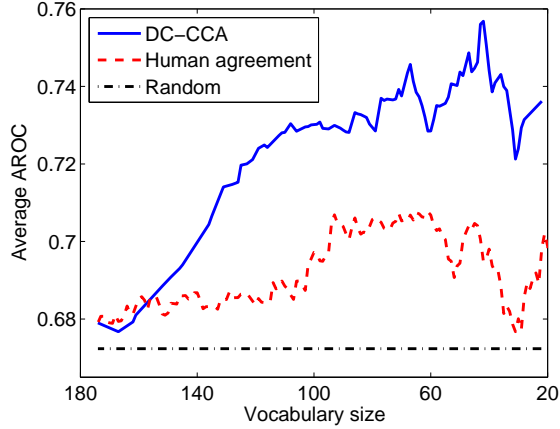}}
  \caption{Comparison of vocabulary selection techniques for music retrieval.}
  \label{fig:aroc2}
\end{figure}
\par In summary, \citet{Torres-07a} illustrates that vocabulary selection using sparse CCA
significantly improves the retrieval performance of a computer audition system (by effectively removing noisy words), outperforming a random baseline and a human agreement heuristic.

\section{Sparse Fisher Discriminant Analysis}\label{Sec:sparseFDA}
In this section, we show that the FDA problem is an \emph{interesting} special case of the GEV problem and that the special structure of $\bm{A}$ allows the sparse FDA problem to be solved more efficiently than the general sparse GEV problem.
\par Let us consider the GEV problem in (GEV-P) with $\bm{A}\in\mathbb{S}^n_+$, $\bm{B}\in\mathbb{S}^n_{++}$ and $\text{rank}(\bm{A})=1$. This is exactly the FDA problem as shown in (\ref{Eq:FDA}) where $\bm{A}$ is of the form $\bm{A}=\bm{a}\bm{a}^T$, with $\bm{a}=(\bm{\mu}_1-\bm{\mu}_2)\in\mathbb{R}^n$. The corresponding GEV problem is written as
\begin{eqnarray}\label{Eq:gevrankone}
\lambda_{max}(\bm{A},\bm{B})=\max_{\bm{x}}&&(\bm{a}^T\bm{x})^2\nonumber\\
\text{s.t.}&&\bm{x}^T\bm{Bx}=1,
\end{eqnarray}
which can also be written as $\lambda_{max}(\bm{A},\bm{B})=\max_{\bm{x}\ne \bm{0}}\frac{(\bm{a}^T\bm{x})^2}{\bm{x}^T\bm{Bx}}$. Since we are primarily interested in the maximizer of (\ref{Eq:gevrankone}), we can rewrite it as 
\begin{equation}\label{Eq:gevmin}
\min_{\bm{x}\ne\bm{0}}\frac{\bm{x}^T\bm{Bx}}{(\bm{a}^T\bm{x})^2}\qquad\equiv\qquad\min\{\bm{x}^T\bm{Bx}\,:\,\bm{a}^T\bm{x}=1\}.
\end{equation}
The advantage of the formulation in (\ref{Eq:gevmin}) will become clear when we consider its sparse version, i.e., after introducing the constraint $\{\bm{x}:\Vert\bm{x}\Vert_0\le k\}$ in (\ref{Eq:gevmin}). Clearly, introducing the sparsity constraint makes the problem intractable. However, introducing an $\ell_1$-norm relaxation in this formulation gives rise to a \emph{convex program}, 
\begin{equation}\label{Eq:sparsefda}
\min\{\bm{x}^T\bm{Bx}\,:\,\bm{a}^T\bm{x}=1,\,\Vert\bm{x}\Vert_1\le k\},
\end{equation}
more specifically a quadratic program (QP), and the corresponding penalized version is given by
\begin{equation}
\min\{\bm{x}^T\bm{Bx}+\nu\Vert\bm{x}\Vert_1\,:\,\bm{a}^T\bm{x}=1\},
\end{equation}
where $\nu>0$ is the regularization parameter.
\par Note that a transformation similar to the one leading to (\ref{Eq:gevmin}) can be performed for the GEV problem with any, general $\bm{A}\in\mathbb{S}^n$, i.e., writing the GEV problem as a minimization problem, 
\begin{eqnarray}\label{Eq:generalgevmin}
\min_{\bm{x}}&&\bm{x}^T\bm{Bx}\nonumber\\
\text{s.t.}&&\bm{x}^T\bm{Ax}=1.
\end{eqnarray}
This formulation, however, is not useful to simplify solving a GEV problem in general. Indeed, consider the sparse version of the problem in (\ref{Eq:generalgevmin}) with the sparsity constraint $\{\bm{x}:\Vert\bm{x}\Vert_0\le k\}$ relaxed to $\{\bm{x}:\Vert\bm{x}\Vert_1\le k\}$. Because of the quadratic equality constraint, the resulting program is non-convex for any $\bm{A}$. Suppose that the constraint set $\{\bm{x}:\bm{x}^T\bm{Ax}=1\}$ is relaxed to $\{\bm{x}:\bm{x}^T\bm{Ax}\le 1\}$. If $\bm{A}\in\mathbb{S}^n\backslash\mathbb{S}^n_+$, the program is still non-convex as the constraint defines a non-convex set. If $\bm{A}\in\mathbb{S}^n_+$, then the optimum occurs at $\bm{x}=\bm{0}$. Therefore, the minimization formulation of the GEV problem in (\ref{Eq:generalgevmin}) is not useful, unlike the case where $\bm{A}\in\mathbb{S}^n_+$ and $\text{rank}(\bm{A})=1$.
\par Based on the discussion so far, it is clear that the sparse FDA problem can be solved as a convex QP, 
which is significantly more efficient than, e.g., an SDP relaxation as in (\ref{Eq:Sparse-sdp-GEV}) for sparse PCA or sparse CCA.
Suppose that one would like to use a better approximation to $\Vert\bm{x}\Vert_0$ than $\Vert\bm{x}\Vert_1$, for sparse FDA. Using the approximation we considered in this work, (\ref{Eq:sparsefda}) reduces to
\begin{eqnarray}\label{Eq:fdastudent}
\min_{\bm{x}}&&\bm{x}^T\bm{Bx}+\nu_\varepsilon\sum^n_{i=1}\log(\varepsilon+|x_i|)\nonumber\\
\text{s.t.}&&\bm{a}^T\bm{x}=1,
\end{eqnarray}
where $\nu_\varepsilon:=\nu/\log(1+\varepsilon^{-1})$. Applying the MM method to the above program results in the following iterative scheme,
\begin{eqnarray}\label{Eq:fdaiterative}
\bm{x}^{(l+1)}=\arg\min_{\bm{x}}&&\bm{x}^T\bm{Bx}+\nu_\varepsilon\sum^n_{i=1}\frac{|x_i|}{|x^{(l)}_i|+\varepsilon}\nonumber\\
\text{s.t.}&&\bm{a}^T\bm{x}= 1,
\end{eqnarray}
which is a sequence of QPs unlike Algorithm~\ref{alg1}, which is a sequence of QCQPs. The nice structure of $\bm{A}$ makes the corresponding sparse GEV problem computationally efficient. Therefore, one should solve the sparse FDA problem by using (\ref{Eq:sparsefda}) or (\ref{Eq:fdaiterative}) instead of using the convex SDP in (\ref{Eq:Sparse-sdp-GEV}) or Algorithm~\ref{alg1}.
\par \citet[Section 3.3]{Suykens-02} and \citet[Proposition 1]{Mika-01} have shown connections between the FDA formulation in (\ref{Eq:gevmin}) with $\bm{a}=\bm{\mu}_1-\bm{\mu}_2$ and $\bm{B}=\bm{\Sigma}_1+\bm{\Sigma}_2$ (see paragraph below (\ref{Eq:FDA}) for details) and least-squares support vector machines (classifiers that minimize the squared loss, see \citet[Chapter 3]{Suykens-02}).
Therefore, sparse FDA is equivalent to feature selection with least-squares support vector machines. In other words, (\ref{Eq:sparsefda}) is equivalent to LASSO, while the formulation in (\ref{Eq:fdastudent}) is similar to the one considered in \cite{Weston-02}. Since these are well studied problems, we do not pursue further showing the numerical performance of sparse FDA.

\section{Conclusion and Discussion}
We study the problem of finding sparse eigenvectors for generalized eigenvalue problems. After proposing a non-convex but tight approximation to the cardinality constraint, we formulate the resulting optimization problem as a d.c. program and derive an iterative algorithm, based on the majorization-minimization method. This results in solving a sequence of quadratically constrained quadratic programs, an algorithm which exhibits global convergence behavior, as we show. We also derive sparse PCA (DC-PCA) and sparse CCA (DC-CCA) algorithms as special cases of our proposed algorithm. Empirical results demonstrate the performance of the proposed algorithm for sparse PCA and sparse CCA applications. In the case of sparse PCA, we experimentally demonstrate on both benchmark and real-life datasets of varying dimensionality that the proposed algorithm (DC-PCA) explains more variance with sparser features than SPCA~\citep{Zou-06} while performing similarly to DSPCA~\citep{Aspremont-07} and GSPCA~\citep{Moghaddam-06a} at better computational speed (lower CPU time). On the other hand, DC-PCA has performance and scalability similar to that of the state-of-the-art GPower$_{\ell_0}$ algorithm.
We also illustrate the practical relevance of the sparse CCA algorithm in two applications: cross-language document retrieval and vocabulary selection for music information retrieval.
\par The proposed algorithm does not allow to set the regularization parameter a priori, to guarantee a given sparsity level. This is similar for SPCA and GPower$_{\ell_0}$. SDP-based relaxation methods, on the other hand (e.g., DSPCA in the context of sparse PCA) are better suited to achieve a given sparsity level in one shot, by incorporating an explicit constraint on the sparsity of the solution (although, eventually, through relaxation, an approximation of the original problem is solved). Since the algorithm we propose solves a LASSO problem in each step but with a quadratic constraint, one could explore using a modified version of path following techniques like least angle regression~\citep{Efron-04} to learn the entire regularization path.
%
\acks{Bharath Sriperumbudur thanks Suvrit Sra for constructive discussions while the former was an intern at the Max Planck Institute for Biological Cybernetics, T\"{u}bingen. The authors wish to acknowledge support from the National Science Foundation (grant DMS-MSPA 0625409), the Fair Isaac Corporation and the University of California MICRO program. 
}

\appendix
\section*{Appendix A. Derivation of the SDP relaxation in (\ref{Eq:Sparse-sdp-GEV})}\label{appendix-a}
The idea in deriving the SDP relaxation in (\ref{Eq:Sparse-sdp-GEV}) is to start with the approximate program in (\ref{Eq:approx}) and then derive its bi-dual (dual of the dual). Though (\ref{Eq:Sparse-sdp-GEV}) is not a canonical convex program, its Lagrangian dual is always convex. Therefore, obtaining the dual program of this dual provides a convex approximation to (\ref{Eq:Sparse-sdp-GEV}), which is what we derive below.
\par Consider the $\ell_1$-norm relaxed sparse GEV problem in (\ref{Eq:approx}), which we reproduce here for convenience.
\begin{eqnarray}\label{Eq:appendix-L1}
\max_{\bm{x}}&&\bm{x}^T\bm{A}\bm{x}\nonumber\\
\text{s.t.}&&\bm{x}^T\bm{B}\bm{x}\le 1,\, \Vert\bm{x}\Vert_1\le k.
\end{eqnarray}
The above problem can be re-written as
\begin{eqnarray}
\max_{\bm{x},\,\bm{y}}&&\bm{x}^T\bm{A}\bm{x}\nonumber\\
\text{s.t.}&&\bm{x}^T\bm{B}\bm{x}\le 1,\,-\bm{y}\preceq \bm{x}\preceq\bm{y}\nonumber\\
&& \bm{y}^T\bm{1}\le k.
\end{eqnarray}
The corresponding Lagrangian dual problem is given by \begin{equation}
\min_{\substack{\beta\ge 0,\mu\ge 0\\ \bm{u}\succeq \bm{0}, \bm{s}\succeq \bm{0}}}\,\,\max_{\substack{\bm{x},\,\bm{y}}}\,\,L(\bm{x},\bm{y},\beta,\mu,\bm{u},\bm{s}),\nonumber\end{equation} where
\begin{equation}\label{Eq:Lagrangian-dual}
L(\bm{x},\bm{y},\beta,\mu,\bm{u},\bm{s})=\bm{x}^T\bm{Ax}-\mu(\bm{x}^T\bm{Bx}-1)-\beta(\bm{y}^T\bm{1}-k)-\bm{u}^T(\bm{x}-\bm{y})+\bm{s}^T(\bm{x}+\bm{y}).
\end{equation}
Let us first maximize $L$ over $\bm{x}$. By Lemma 3.6 of \cite{Lemarechal-99}, the necessary and sufficient condition for $Q(\bm{x})=\bm{x}^T(\bm{A}-\mu\bm{B})\bm{x}+\bm{x}^T(\bm{s}-\bm{u})$ to have a finite upper bound over $\mathbb{R}^n$ is $\mu\bm{B}-\bm{A}\succeq 0$ and $\bm{s}-\bm{u}\in\mathcal{R}(\mu\bm{B}-\bm{A})$. Differentiating $L$ w.r.t. $\bm{x}$ yields $\bm{x}=\frac{1}{2}(\mu\bm{B}-\bm{A})^\dagger(\bm{s}-\bm{u})$. Similarly, while maximizing $L$ w.r.t. $\bm{y}$, the necessary and sufficient condition for $R(\bm{y})=\bm{y}^T(\bm{s}+\bm{u}-\beta\bm{1})$ to have a finite upper bound over $\mathbb{R}^n$ is $\bm{s}+\bm{u}=\beta\bm{1}$. Therefore, the dual program can be written as
\begin{eqnarray}
\min_{\bm{u},\bm{s},\beta,\mu}&&\frac{1}{4}(\bm{u}-\bm{s})^T(\mu\bm{B}-\bm{A})^\dagger(\bm{u}-\bm{s})+\beta k+\mu\nonumber\\
\text{s.t.}&& \mu\bm{B}-\bm{A}\succeq 0,\,\bm{u}-\bm{s}\in\mathcal{R}(\mu\bm{B}-\bm{A})\nonumber\\
&& \bm{s}+\bm{u}=\beta\bm{1},\,\beta\ge 0,\,\mu\ge 0,\,\bm{u}\succeq\bm{0},\,\bm{s}\succeq\bm{0},
\end{eqnarray}
which is equivalent to
\begin{eqnarray}
\min_{\bm{r},\beta,\mu}&&\frac{1}{4}\bm{r}^T(\mu\bm{B}-\bm{A})^\dagger\bm{r}+\beta k+\mu\nonumber\\
\text{s.t.}&& \mu\bm{B}-\bm{A}\succeq 0,\,\bm{r}\in\mathcal{R}(\mu\bm{B}-\bm{A})\nonumber\\
&& -\beta\bm{1}\preceq\bm{r}\preceq\beta\bm{1},\,\beta\ge 0,\,\mu\ge 0.
\end{eqnarray}
By invoking the Schur's complement lemma, the dual can be written as
\begin{eqnarray}\label{Eq:appendix-dual}
\min_{\bm{r},t,\beta,\mu}&&t+\beta k+\mu\nonumber\\
\text{s.t.}&& -\beta\bm{1}\preceq\bm{r}\preceq\beta\bm{1},\,\beta\ge 0,\,\mu\ge 0\nonumber\\
&&\left(\begin{array}{cc}
\mu\bm{B}-\bm{A} & -\frac{1}{2}\bm{r} \\
-\frac{1}{2}\bm{r}^T & t \end{array}\right)\succeq 0.
\end{eqnarray}
The bi-dual associated with (\ref{Eq:appendix-L1}) is obtained by computing the dual of (\ref{Eq:appendix-dual}) given by
\begin{equation}
\max_{\substack{\phi\in\mathbb{R},\alpha\ge 0,\theta\ge 0\\ \bm{\tau}\succeq\bm{0},\bm{\eta}\succeq\bm{0},\bm{x}\succeq\bm{0}\\ \bm{X}\succeq 0}}\,\,\min_{\substack{\bm{r}\succeq \bm{0},t\in\mathbb{R}\\\beta\ge 0,\mu\ge 0}}\,\,\tilde{L}(\bm{r},t,\beta,\mu,\phi,\alpha,\theta,\bm{\tau},\bm{X},\bm{x},\bm{\eta}).
\end{equation}
Here $\tilde{L}$ is the Lagrangian associated with (\ref{Eq:appendix-dual}), given by
\begin{eqnarray}
\tilde{L}(\bm{r},t,\beta,\mu,\phi,\alpha,\theta,\bm{\tau},\bm{X},\bm{x},\bm{\eta})&=&t+\beta k+\mu+\bm{\eta}^T(\bm{r}-\beta\bm{1})-\bm{\tau}^T(\bm{r}+\beta\bm{1})-\alpha\mu-\theta\beta\nonumber\\
&&-\text{tr}\left[\left(\begin{array}{cc}
\bm{X} & \bm{x} \\
\bm{x}^T & \phi \end{array}\right)\left(\begin{array}{cc}
\mu\bm{B}-\bm{A} & -\frac{1}{2}\bm{r} \\
-\frac{1}{2}\bm{r}^T & t \end{array}\right)\right]\nonumber\\
&=& \text{tr}(\bm{XA})+\mu(1-\alpha-\text{tr}(\bm{XB}))+t(1-\phi)\nonumber\\
&&+\beta(k-\bm{\eta}^T\bm{1}-\bm{\tau}^T\bm{1}-\theta)+\bm{r}^T(\bm{\eta}-\bm{\tau}+\bm{x}).
\end{eqnarray}
Minimizing the above Lagrangian results in
\begin{eqnarray}
\max_{\alpha,\theta,\bm{\tau},\bm{\eta},\bm{x},\bm{X}}&&\text{tr}(\bm{XA})\nonumber\\
\text{s.t.}&& \alpha+\text{tr}(\bm{XB})=1,\,\,\bm{x}+\bm{\eta}=\bm{\tau},\,(\bm{\eta}+\bm{\tau})^T\bm{1}+\theta=k\nonumber\\
&&\left(\begin{array}{cc}
\bm{X} & \bm{x} \\
\bm{x}^T & 1 \end{array}\right)\succeq 0,
\end{eqnarray}
which is equivalent to
\begin{eqnarray}
\max_{\bm{x},\bm{X}}&&\text{tr}(\bm{XA})\nonumber\\
\text{s.t.}&& \text{tr}(\bm{XB})\le 1,\Vert \bm{x}\Vert_1\le k\nonumber\\
&&\left(\begin{array}{cc}
\bm{X} & \bm{x} \\
\bm{x}^T & 1 \end{array}\right)\succeq 0,
\end{eqnarray}
as shown in (\ref{Eq:Sparse-sdp-GEV}).

\section*{Appendix B. Alternative derivation of (ALG)}\label{appendix-b}
(ALG) can be derived differently by starting with (\ref{Eq:tau-1}) and applying the linear majorization idea (see Example~\ref{Exm:dc}).
\par Consider the d.c. program in (\ref{Eq:tau-1}), which is of the form $\min_{\bm{x},\bm{y}}(u(\bm{x},\bm{y})-v(\bm{x},\bm{y}))$ where $u(\bm{x},\bm{y})=I_{\Omega}(\bm{x},\bm{y})+\tau\Vert\bm{x}\Vert^2_2$ and $v(\bm{x},\bm{y})=\bm{x}^T(\bm{A}+\tau\bm{I}_n)\bm{x}-\rho_\varepsilon\sum^n_{i=1}\log(y_i+\varepsilon)$ with $\Omega=\{(\bm{x},\bm{y}):\bm{x}^T\bm{Bx}\le 1,\,-\bm{y}\preceq\bm{x}\preceq\bm{y}\}$. Here $I_\Omega$ represents the indicator function of the convex set $\Omega$ given by
\begin{equation}
I_{\Omega}(\bm{x},\bm{y})=\left\{\begin{array}{c@{\quad\quad}l}
0,& (\bm{x},\bm{y})\in\Omega\\
\infty,& \text{otherwise}
\end{array}\right..\nonumber
\label{Eq:indicator}
\end{equation} 
It is easy to check that $u$ and $v$ are convex. Therefore, by (\ref{Eq:example}) in Example~\ref{Exm:dc}, the MM algorithm gives
\begin{eqnarray}\label{Eq:qcqp-0}
(\bm{x}^{(l+1)},\bm{y}^{(l+1)})=\arg\min_{\bm{x},\bm{y}}&&\tau\Vert\bm{x}\Vert^2_2-2\bm{x}^T(\bm{A}+\tau\bm{I}_n)\bm{x}^{(l)}+\rho_\varepsilon\sum^n_{i=1}\frac{y_i}{y^{(l)}_i+\varepsilon}\nonumber\\
\text{s.t.}&&\bm{x}^T\bm{Bx}\le 1,\,-\bm{y}\preceq\bm{x}\preceq\bm{y},
\end{eqnarray}
which is equivalent to (ALG).
\section*{Appendix C. Derivation of (ALG-S)}\label{appendix-c}
Suppose $\bm{A}\succeq 0$ and $\bm{B}=\bm{I}_n$. Since $\bm{A}\succeq 0$, $\tau$ can be chosen as zero. Using $\tau=0$ in (ALG), we have that $\bm{x}^{(l+1)}$ is the maximizer of the following program:
\begin{equation}\label{Eq:simpleiterative}
\max_{\bm{x}^T\bm{x}\le 1}\,\,\bm{x}^T\bm{Ax}^{(l)}-\frac{\rho_\varepsilon}{2}\left\Vert\bm{W}^{(l)}\bm{x}\right\Vert_1=
\max_{\bm{x}^T\bm{x}\le 1}\,\,\sum^n_{i=1}x_i(\bm{Ax}^{(l)})_i-\frac{\rho_\varepsilon}{2}w^{(l)}_i|x_i|.
\end{equation}
Consider the r.h.s. of (\ref{Eq:simpleiterative}). Since it is the maximization of a linear objective over a convex set, the unique optimum lies on the boundary of the convex set \cite[Theorem 32.1]{Rockafeller-70}. The Lagrangian associated with the program in the r.h.s. of (\ref{Eq:simpleiterative}) is given by
\begin{equation}\label{Eq:lagrange-c}
L(\bm{x},\lambda)=\sum^n_{i=1}x_i(\bm{Ax}^{(l)})_i-\frac{\rho_\varepsilon}{2}w^{(l)}_i|x_i|-\lambda\sum^n_{i=1}x^2_i,
\end{equation}
where $\lambda>0$. Differentiating $L$ w.r.t. $x_i$ and setting it to zero yields
\begin{equation}
x_i=\frac{(\bm{Ax}^{(l)})_i-\frac{\rho_\varepsilon}{2}w^{(l)}_i\text{sign}(x_i)}{2\lambda}.
\end{equation}
Therefore, we have
\begin{equation}
x_i=\left\{\begin{array}{cc}
\frac{(\bm{Ax}^{(l)})_i-\frac{\rho_\varepsilon}{2}w^{(l)}_i}{2\lambda}, & (\bm{Ax}^{(l)})_i\ge\frac{\rho_\varepsilon}{2}w^{(l)}_i\\
\frac{(\bm{Ax}^{(l)})_i+\frac{\rho_\varepsilon}{2}w^{(l)}_i}{2\lambda}, & (\bm{Ax}^{(l)})_i\le-\frac{\rho_\varepsilon}{2}w^{(l)}_i\\
0, & \text{otherwise}\end{array}\right.,
\end{equation}
which is equivalently written as 
\begin{equation}\label{Eq:x_i}
x_i=\frac{\left[\left|(\bm{Ax}^{(l)})_i\right|-\frac{\rho_\varepsilon}{2}w^{(l)}_i\right]_+\text{sign}((\bm{Ax}^{(l)})_i)}{2\lambda}.
\end{equation}
Since $\sum^n_{i=1}x^2_i=1$, substituting for $x_i$ as given in (\ref{Eq:x_i}) yields
\begin{equation}\label{Eq:x_i-final}
x_i=\frac{\left[\left|(\bm{Ax}^{(l)})_i\right|-\frac{\rho_\varepsilon}{2}w^{(l)}_i\right]_+\text{sign}((\bm{Ax}^{(l)})_i)}{\sqrt{\sum^n_{i=1}\left[\left|(\bm{Ax}^{(l)})_i\right|-\frac{\rho_\varepsilon}{2}w^{(l)}_i\right]^2_+}},\nonumber
\end{equation}
and therefore $\bm{x}^{(l+1)}$ in (ALG-S) follows.


\begin{thebibliography}{60}
\providecommand{\natexlab}[1]{#1}
\providecommand{\url}[1]{\texttt{#1}}
\expandafter\ifx\csname urlstyle\endcsname\relax
  \providecommand{\doi}[1]{doi: #1}\else
  \providecommand{\doi}{doi: \begingroup \urlstyle{rm}\Url}\fi

\bibitem[Alon et~al.(1999)Alon, Barkai, Notterman, Gish, Ybarra, Mack, and
  Levine]{Alon-99}
U.~Alon, N.~Barkai, D.~A. Notterman, K.~Gish, S.~Ybarra, D.~Mack, and A.~J.
  Levine.
\newblock Broad patterns of gene expression revealed by clustering analysis of
  tumor and normal colon cancer tissues.
\newblock \emph{Cell Biology}, 96:\penalty0 6745--6750, 1999.

\bibitem[B{\"{o}}hning and Lindsay(1988)]{Bohning-88}
D.~B{\"{o}}hning and B.~G. Lindsay.
\newblock Monotonicity of quadratic-approximation algorithms.
\newblock \emph{Annals of the Institute of Statistical Mathematics},
  40\penalty0 (4):\penalty0 641--663, 1988.

\bibitem[Bonnans et~al.(2006)Bonnans, Gilbert, Lemar{\'{e}}chal, and
  Sagastiz{\'{a}}bal]{Bonnans-06}
J.~F. Bonnans, J.~C. Gilbert, C.~Lemar{\'{e}}chal, and C.~A.
  Sagastiz{\'{a}}bal.
\newblock \emph{Numerical Optimization: Theoretical and Practical Aspects}.
\newblock Springer-Verlag, 2006.

\bibitem[Boyd and Vandenberghe(2004)]{Boyd-06}
S.~P. Boyd and L.~Vandenberghe.
\newblock \emph{Convex {O}ptimization}.
\newblock Cambridge University Press, 2004.

\bibitem[Bradley and Mangasarian(1998)]{Bradley-98}
P.~S. Bradley and O.~L. Mangasarian.
\newblock Feature selection via concave minimization and support vector
  machines.
\newblock In \emph{Proc. 15th International Conf. on Machine Learning}, pages
  82--90. Morgan Kaufmann, San Francisco, CA, 1998.

\bibitem[Cadima and Jolliffe(1995)]{Cadima-95}
J.~Cadima and I.~Jolliffe.
\newblock Loadings and correlations in the interpretation of principal
  components.
\newblock \emph{Applied Statistics}, 22:\penalty0 203--214, 1995.

\bibitem[Candes et~al.(2007)Candes, Wakin, and Boyd]{Candes-07}
E.~J. Candes, M.~Wakin, and S.~Boyd.
\newblock Enhancing sparsity by reweighted $\ell_1$ minimization.
\newblock \emph{J. Fourier Anal. Appl.}, 2007.
\newblock To appear.

\bibitem[d'Aspremont et~al.(2005)d'Aspremont, El\mbox{ }Ghaoui, Jordan, and
  Lanckriet]{Aspremont-04}
A.~d'Aspremont, L.~El\mbox{ }Ghaoui, M.~I. Jordan, and G.~R.~G. Lanckriet.
\newblock A direct formulation for sparse {PCA} using semidefinite programming.
\newblock In Lawrence~K. Saul, Yair Weiss, and {L\'{e}on} Bottou, editors,
  \emph{Advances in Neural Information Processing Systems 17}, pages 41--48,
  Cambridge, MA, 2005. MIT Press.

\bibitem[d'Aspremont et~al.(2007)d'Aspremont, El\mbox{ }Ghaoui, Jordan, and
  Lanckriet]{Aspremont-07}
A.~d'Aspremont, L.~El\mbox{ }Ghaoui, M.~I. Jordan, and G.~R.~G. Lanckriet.
\newblock A direct formulation for sparse {PCA} using semidefinite programming.
\newblock \emph{SIAM Review}, 49\penalty0 (3):\penalty0 434--448, 2007.

\bibitem[d'Aspremont et~al.(2008)d'Aspremont, Bach, and El\mbox{
  }Ghaoui]{Aspremont-08}
A.~d'Aspremont, F.~R. Bach, and L.~El\mbox{ }Ghaoui.
\newblock Optimal solutions for sparse principal component analysis.
\newblock \emph{Journal of Machine Learning Research}, 9:\penalty0 1269--1294,
  2008.

\bibitem[deLeeuw(1977)]{deLeeuw-77}
J.~deLeeuw.
\newblock Applications of convex analysis to multidimensional scaling.
\newblock In J.~R. Barra, F.~Brodeau, G.~Romier, and B.~Van Cutsem, editors,
  \emph{Recent advantages in Statistics}, pages 133--146, Amsterdam, The
  Netherlands, 1977. North Holland Publishing Company.

\bibitem[Dempster et~al.(1977)Dempster, Laird, and Rubin]{Dempster-77}
A.~P. Dempster, N.~M. Laird, and D.~B. Rubin.
\newblock Maximum likelihood from incomplete data via the {EM} algorithm.
\newblock \emph{J. Roy. Stat. Soc. B}, 39:\penalty0 1--38, 1977.

\bibitem[Efron et~al.(2004)Efron, Hastie, Johnstone, and Tibshirani]{Efron-04}
B.~Efron, T.~Hastie, I.~Johnstone, and R.~Tibshirani.
\newblock Least angle regression.
\newblock \emph{Annals of Statistics}, 32\penalty0 (2):\penalty0 407--499,
  2004.

\bibitem[Fazel et~al.(2003)Fazel, Hindi, and Boyd]{Fazel-03}
M.~Fazel, H.~Hindi, and S.~Boyd.
\newblock Log-det heuristic for matrix rank minimization with applications to
  {H}ankel and {E}uclidean distance matrices.
\newblock In \emph{Proc. American Control Conference}, Denver, Colorado, 2003.

\bibitem[Germann(2001)]{Germann-01}
U.~Germann.
\newblock Aligned {H}ansards of the $36^{th}$ parliament of {C}anada, 2001.
\newblock http://www.isi.edu/natural-language/download/hansard/.

\bibitem[Golub et~al.(1999)Golub, Slonim, Tamayo, Huard, Gaasenbeek, Mesirov,
  Coller, Loh, Downing, Caligiuri, Bloomfield, and Lander]{Golub-99}
T.~R. Golub, D.~K. Slonim, P.~Tamayo, C.~Huard, M.~Gaasenbeek, J.P. Mesirov,
  H.~Coller, M.~K. Loh, J.~R. Downing, M.~A. Caligiuri, C.~D. Bloomfield, and
  E.S. Lander.
\newblock Molecular classification of cancer: {C}lass discovery and class
  prediction by gene expression monitoring.
\newblock \emph{Science}, 286:\penalty0 531--537, October 1999.

\bibitem[Hardoon and Shawe-Taylor(2008)]{Hardoon-08}
D.~R. Hardoon and J.~Shawe-Taylor.
\newblock Sparse {CCA} for bilingual word generation.
\newblock In \emph{EURO Mini Conference, Continuous Optimization and
  Knowledge-Based Technologies}, 2008.

\bibitem[Heiser(1987)]{Heiser-87}
W.~J. Heiser.
\newblock Correspondence analysis with least absolute residuals.
\newblock \emph{Comput. Stat. Data Analysis}, 5:\penalty0 337--356, 1987.

\bibitem[Horst and Thoai(1999)]{Horst-99}
R.~Horst and N.~V. Thoai.
\newblock D.c. programming: {O}verview.
\newblock \emph{Journal of Optimization Theory and Applications}, 103:\penalty0
  1--43, 1999.

\bibitem[Hotelling(1933)]{Hotelling-33}
H.~Hotelling.
\newblock Analysis of a complex of statistical variables into principal
  components.
\newblock \emph{Journal of Educational Psychology}, 24:\penalty0 417--441,
  1933.

\bibitem[Hotelling(1936)]{Hotelling-36}
H.~Hotelling.
\newblock Relations between two sets of variates.
\newblock \emph{Biometrika}, 28:\penalty0 321--377, 1936.

\bibitem[Huber(1981)]{Huber-81}
P.~J. Huber.
\newblock \emph{Robust Statistics}.
\newblock John Wiley, New York, 1981.

\bibitem[Hunter and Lange(2004)]{Hunter-04}
D.~R. Hunter and K.~Lange.
\newblock A tutorial on {MM} algorithms.
\newblock \emph{The American Statistician}, 58:\penalty0 30--37, 2004.

\bibitem[Hunter and Li(2005)]{Hunter-05}
D.~R. Hunter and R.~Li.
\newblock Variable selection using {MM} algorithms.
\newblock \emph{The Annals of Statistics}, 33:\penalty0 1617--1642, 2005.

\bibitem[Jeffers(1967)]{Jeffers-67}
J.~Jeffers.
\newblock Two case studies in the application of principal components.
\newblock \emph{Applied Statistics}, 16:\penalty0 225--236, 1967.

\bibitem[Jolliffe(1986)]{Jollife-86}
I.~Jolliffe.
\newblock \emph{Principal component analysis}.
\newblock Springer-Verlag, New York, USA, 1986.

\bibitem[Jolliffe et~al.(2003)Jolliffe, Trendafilov, and Uddin]{Jolliffe-03}
I.~T. Jolliffe, N.~T. Trendafilov, and M.~Uddin.
\newblock A modified principal component technique based on the {LASSO}.
\newblock \emph{Journal of Computational and Graphical Statistics},
  12:\penalty0 531--547, 2003.

\bibitem[Journ{\'{e}}e et~al.(2008)Journ{\'{e}}e, Nesterov, Richt{\'{a}}rik,
  and Sepulchre]{Journee-08}
M.~Journ{\'{e}}e, Y.~Nesterov, P.~Richt{\'{a}}rik, and R.~Sepulchre.
\newblock Generalized power method for sparse principal component analysis.
\newblock \emph{http://arxiv.org/abs/0811.4724v1}, November 2008.

\bibitem[Lange et~al.(2000)Lange, Hunter, and Yang]{Lange-00}
K.~Lange, D.~R. Hunter, and I.~Yang.
\newblock Optimization transfer using surrogate objective functions with
  discussion.
\newblock \emph{Journal of Computational and Graphical Statistics}, 9\penalty0
  (1):\penalty0 1--59, 2000.

\bibitem[Lee and Seung(2001)]{LeeSeung-01}
D.~D. Lee and H.~S. Seung.
\newblock Algorithms for non-negative matrix factorization.
\newblock In T.K. Leen, T.G. Dietterich, and V.~Tresp, editors, \emph{Advances
  in Neural Information Processing Systems 13}, pages 556--562. MIT Press,
  Cambridge, 2001.

\bibitem[Lemar{\'e}chal and Oustry(1999)]{Lemarechal-99}
C.~Lemar{\'e}chal and F.~Oustry.
\newblock Semidefinite relaxations and {L}agrangian duality with application to
  combinatorial optimization.
\newblock Technical Report RR3710, INRIA, 1999.

\bibitem[Littman et~al.(1998)Littman, Dumais, and Landauer]{Littman-98}
M.~L. Littman, S.~T. Dumais, and T.~K. Landauer.
\newblock Automatic cross-language information retrieval using latent semantic
  indexing.
\newblock In G.~Grefenstette, editor, \emph{Cross-Language Information
  Retrieval}, pages 51--62. Kluwer Academic Publishers, 1998.

\bibitem[Mackey(2009)]{Mackey-08}
L.~Mackey.
\newblock Deflation methods for sparse pca.
\newblock In D.~Koller, D.~Schuurmans, Y.~Bengio, and L.~Bottou, editors,
  \emph{Advances in Neural Information Processing Systems 21}, pages
  1017--1024. MIT Press, 2009.

\bibitem[McCabe(1984)]{McCabe-84}
G.~McCabe.
\newblock Principal variables.
\newblock \emph{Technometrics}, 26:\penalty0 137--144, 1984.

\bibitem[Mckinney(2003)]{Mckinney-03}
M.~F. Mckinney.
\newblock Features for audio and music classification.
\newblock In \emph{Proc. of the International Symposium on Music Information
  Retrieval}, pages 151--158, 2003.

\bibitem[Meng(2000)]{Meng-00}
X.-L. Meng.
\newblock Discussion on ``optimization transfer using surrogate objective
  functions''.
\newblock \emph{Journal of Computational and Graphical Statistics}, 9\penalty0
  (1):\penalty0 35--43, 2000.

\bibitem[Mika et~al.(2001)Mika, R{\"a}tsch, and M{\"u}ller]{Mika-01}
S.~Mika, G.~R{\"a}tsch, and K.-R. M{\"u}ller.
\newblock A mathematical programming approach to the kernel {F}isher algorithm.
\newblock In T.K. Leen, T.G. Dietterich, and V.~Tresp, editors, \emph{Advances
  in Neural Information Processing Systems 13}, Cambridge, MA, 2001. MIT Press.

\bibitem[Minoux(1986)]{Minoux-86}
M.~Minoux.
\newblock \emph{Mathematical Programming: Theory and Algorithms}.
\newblock John Wiley \& Sons Ltd., 1986.

\bibitem[Moghaddam et~al.(2007)Moghaddam, Weiss, and Avidan]{Moghaddam-06a}
B.~Moghaddam, Y.~Weiss, and S.~Avidan.
\newblock Spectral bounds for sparse {PCA}: {E}xact and greedy algorithms.
\newblock In B.~Sch\"{o}lkopf, J.~Platt, and T.~Hoffman, editors,
  \emph{Advances in Neural Information Processing Systems 19}, Cambridge, MA,
  2007. MIT Press.

\bibitem[Nesterov(2005)]{Nesterov-05}
Y.~Nesterov.
\newblock Smooth minimization of non-smooth functions.
\newblock \emph{Mathematical Programming, Series A}, 103:\penalty0 127--152,
  2005.

\bibitem[Ortega and Rheinboldt(1970)]{Ortega-70}
J.~M. Ortega and W.~C. Rheinboldt.
\newblock \emph{Iterative Solution of Nonlinear Equations in Several
  Variables}.
\newblock Academic Press, New York, 1970.

\bibitem[Ramaswamy et~al.(2001)Ramaswamy, Tamayo, Rifkin, Mukherjee, Yeang,
  Angelo, Ladd, Reich, Latulippe, Mesirov, Poggio, Gerald, Loda, Lander, and
  Golub]{Ramaswamy-01}
S.~Ramaswamy, P.~Tamayo, R.~Rifkin, S.~Mukherjee, C.~Yeang, M.~Angelo, C.~Ladd,
  M.~Reich, E.~Latulippe, J.~Mesirov, T.~Poggio, W.~Gerald, M.~Loda, E.~Lander,
  and T.~Golub.
\newblock Multiclass cancer diagnosis using tumor gene expression signature.
\newblock \emph{Proceedings of the National Academy of Sciences}, 98:\penalty0
  15149--15154, 2001.

\bibitem[Rockafellar(1970)]{Rockafeller-70}
R.~T. Rockafellar.
\newblock \emph{Convex Analysis}.
\newblock Princeton University Press, Princeton, NJ, 1970.

\bibitem[Shawe-Taylor and Christianini(2004)]{ShaweTaylor-04}
J.~Shawe-Taylor and N.~Christianini.
\newblock \emph{Kernel Methods for Pattern Analysis}.
\newblock Cambridge University Press, 2004.

\bibitem[Sriperumbudur and Lanckriet(2009)]{Sriperumbudur-09a}
B.~K. Sriperumbudur and G.~R.~G. Lanckriet.
\newblock On the convergence of the concave-convex procedure.
\newblock In \emph{NIPS}, 2009.
\newblock To appear.

\bibitem[Sriperumbudur et~al.(2007)Sriperumbudur, Torres, and
  Lanckriet]{Sriperumbudur-07b}
B.~K. Sriperumbudur, D.~A. Torres, and G.~R.~G. Lanckriet.
\newblock Sparse eigen methods by d.c. programming.
\newblock In \emph{Proc. of the 24$^{th}$ Annual International Conference on
  Machine Learning}, 2007.

\bibitem[Strang(1986)]{Strang-86}
G.~Strang.
\newblock \emph{Introduction to Applied Mathematics}.
\newblock Wellesley-Cambridge Press, 1986.

\bibitem[Suykens et~al.(2002)Suykens, Gestel, Brabanter, Moor, and
  Vandewalle]{Suykens-02}
J.~A.~K. Suykens, T.~Van Gestel, J.~De Brabanter, B.~De Moor, and
  J.~Vandewalle.
\newblock \emph{Least Squares Support Vector Machines}.
\newblock World Scientific Publishing, Singapore, 2002.

\bibitem[Tibshirani(1996)]{Tibshirani-96}
R.~Tibshirani.
\newblock Regression shrinkage and selection via the {LASSO}.
\newblock \emph{Journal of Royal Statistical Society, Series B}, 58\penalty0
  (1):\penalty0 267--288, 1996.

\bibitem[Tipping(2001)]{Tipping-01}
M.~E. Tipping.
\newblock Sparse {B}ayesian learning and the relevance vector machine.
\newblock \emph{Journal of Machine Learning Research}, 1:\penalty0 211--244,
  2001.

\bibitem[Torres et~al.(2007{\natexlab{a}})Torres, Turnbull, Barrington, and
  Lanckriet]{Torres-07a}
D.~Torres, D.~Turnbull, L.~Barrington, and G.~R.~G. Lanckriet.
\newblock Identifying words that are musically meaningful.
\newblock In \emph{Proc. of International Symposium on Music Information and
  Retrieval}, 2007{\natexlab{a}}.

\bibitem[Torres et~al.(2007{\natexlab{b}})Torres, Turnbull, Sriperumbudur,
  Barrington, and Lanckriet]{Torres-07b}
D.~A. Torres, D.~Turnbull, B.~K. Sriperumbudur, L.~Barrington, and G.~R.~G.
  Lanckriet.
\newblock Finding musically meaningful words using sparse {CCA}.
\newblock In \emph{Music, Brain \& Cognition Workshop, NIPS},
  2007{\natexlab{b}}.

\bibitem[Turnbull et~al.(2008)Turnbull, Barrington, Torres, and
  Lanckriet]{Turnbull-08}
D.~Turnbull, L.~Barrington, D.~Torres, and G.~R.~G. Lanckriet.
\newblock Semantic annotation and retrieval of music and sound effects.
\newblock \emph{IEEE Trans. on Audio, Speech and Language Processing},
  16:\penalty0 467--476, 2008.

\bibitem[Vandenberghe and Boyd(1996)]{Vandenberghe-96}
L.~Vandenberghe and S.~Boyd.
\newblock Semidefinite programming.
\newblock \emph{SIAM Review}, 38:\penalty0 49--95, 1996.

\bibitem[Vinokourov et~al.(2003)Vinokourov, Shawe-Taylor, and
  Cristianini]{Vinokourov-03}
A.~Vinokourov, J.~Shawe-Taylor, and N.~Cristianini.
\newblock Inferring a semantic representation of text via cross-language
  correlation analysis.
\newblock In S.~Becker, S.~Thrun, and K.~Obermayer, editors, \emph{Advances in
  Neural Information Processing Systems 15}, pages 1473--1480, Cambridge, MA,
  2003. MIT Press.

\bibitem[Weston et~al.(2003)Weston, Elisseeff, Sch\"{o}lkopf, and
  Tipping]{Weston-02}
J.~Weston, A.~Elisseeff, B.~Sch\"{o}lkopf, and M.~Tipping.
\newblock Use of the zero-norm with linear models and kernel methods.
\newblock \emph{Journal of Machine Learning Research}, 3:\penalty0 1439--1461,
  March 2003.

\bibitem[Yuille and Rangarajan(2003)]{Yuille-03}
A.~L. Yuille and A.~Rangarajan.
\newblock The concave-convex procedure.
\newblock \emph{Neural Computation}, 15:\penalty0 915--936, 2003.

\bibitem[Zangwill(1969)]{Zangwill-69}
W.~I. Zangwill.
\newblock \emph{Nonlinear Programming: A Unified Approach}.
\newblock Prentice-Hall, Englewood Cliffs, N.J., 1969.

\bibitem[Zou and Hastie(2005)]{Zou-05}
H.~Zou and T.~Hastie.
\newblock Regularization and variable selection via the elastic net.
\newblock \emph{J. R. Statist. Soc. B}, 67:\penalty0 301--320, 2005.

\bibitem[Zou et~al.(2006)Zou, Hastie, and Tibshirani]{Zou-06}
H.~Zou, T.~Hastie, and R.~Tibshirani.
\newblock Sparse principal component analysis.
\newblock \emph{Journal of Computational and Graphical Statistics},
  15:\penalty0 265--286, 2006.

\end{thebibliography}
\end{document}